\documentclass[10pt]{article}
\usepackage[compress, round]{natbib}
\usepackage{ss}
\usepackage[utf8]{inputenc}
\usepackage[T1]{fontenc}

\usepackage[verbose=true,
letterpaper,
textheight=8.7in,
textwidth=6.4in,
headheight=12pt,
headsep=25pt,
footskip=30pt]{geometry}

\usepackage{mathtools}
\usepackage{booktabs}
\usepackage{graphicx}
\usepackage{amsmath,amssymb,amsfonts,amsthm,amsxtra,bm}
\usepackage{amsthm}
\usepackage{xspace}
\usepackage{enumerate}
\usepackage{hyperref}
\usepackage{url}
\usepackage{colortbl}
\usepackage{makecell,multirow}       % professional-quality tables
\usepackage{nicefrac}       % compact symbols for 1/2, etc.
\usepackage{thmtools}  
\usepackage{xspace}
\usepackage{footnote, tablefootnote}
\usepackage{float}
\floatstyle{plaintop}
\restylefloat{table}
\usepackage{epstopdf}
\usepackage{latexsym}
\usepackage{algorithm, algorithmicx, algpseudocode}
\usepackage{todonotes}
\usepackage{bbm}

\numberwithin{equation}{section}
\definecolor{darkblue}{rgb}{0.0,0.0,0.65}
\definecolor{darkred}{rgb}{0.68,0.05,0.0}
\definecolor{darkgreen}{rgb}{0.0,0.29,0.29}
\definecolor{darkpurple}{rgb}{0.47,0.09,0.29}

\hypersetup{
   colorlinks = true,
   citecolor  = darkblue,
   linkcolor  = darkred,
   filecolor  = darkblue,
   urlcolor   = darkblue,
 }

\newtheorem{theorem}{Theorem}
\newtheorem{lemma}{Lemma}

\newtheorem{corollary}[theorem]{Corollary}
\newtheorem{definition}{Definition}
\newtheorem{rmk}{Remark}

\newtheorem{innercustomcor}{Corollary}
\newenvironment{customcor}[1]
  {\renewcommand\theinnercustomcor{#1}\innercustomcor}
  {\endinnercustomcor}

\newtheorem{innercustomlem}{Lemma}
\newenvironment{customlem}[1]
  {\renewcommand\theinnercustomlem{#1}\innercustomlem}
  {\endinnercustomlem}

\newcommand{\nashvsp}{\textsc{V-sp}}
\newcommand{\nashvol}{\textsc{V-ol}}
\newcommand{\nashqsp}{\textsc{Q-sp}}
\newcommand{\nashqol}{\textsc{Q-ol}}

\newcommand{\nlsum}{\sum\nolimits}
\newcommand{\nlprod}{\prod\nolimits}

\newcommand{\bigo}{\mathcal{O}}
\newcommand{\bigotilde}{\tilde{\mathcal{O}}}

\newcommand{\mg}{\mathrm{MG}}
\newcommand{\Regret}{\textnormal{Regret}}
\newcommand{\Trans}{\mathbb{P}}
\newcommand{\uV}{V}
\newcommand{\uQ}{Q}
\newcommand{\oV}[1]{{V_{#1}^{\ast}}}
\newcommand{\oQ}[1]{Q_{#1}^{\ast}}
\newcommand{\uL}{L}
\newcommand{\ul}{l}
\newcommand{\ubeta}{\beta}
\newcommand{\ueta}{\eta}
\newcommand{\ugamma}{\overline{\gamma}}

\newcommand{\hgamma}{\hat{\gamma}}
\newcommand{\tbeta}{\tilde{\beta}}

\newcommand{\setto}{\leftarrow}

\newcommand{\D}{\mathbb{D}}
\newcommand{\E}{\mathbb{E}}
\newcommand{\I}{\mathbb{I}}
\newcommand{\N}{\mathbb{N}}
\DeclareMathAlphabet{\mathsfit}{T1}{\sfdefault}{\mddefault}{\sldefault}
\SetMathAlphabet{\mathsfit}{bold}{T1}{\sfdefault}{\bfdefault}{\sldefault}

\newcommand{\brac}[1]{{\left[ #1 \right]}}

\newcommand{\cO}{\mathcal{O}}

\newcommand{\Ber}{{\sf Ber}}

\newcommand{\set}[1]{{\left\{ #1 \right\}}}

\newcommand{\eps}{\varepsilon}

\newcommand{\cA}{\mathcal{A}}
\newcommand{\cB}{\mathcal{B}}

\newcommand{\cM}{\mathcal{M}}

\newcommand{\cS}{\mathcal{S}}

\newcommand{\cX}{\mathcal{X}}

\newcommand{\defeq}{\mathrel{\mathop:}=}

\makeatletter
\newcommand{\vast}{\bBigg@{3}}
\newcommand{\Vast}{\bBigg@{3.5}}
\makeatother

\graphicspath{{./figs/}}

\begin{document}

\title{Online Learning in Unknown Markov Games}

\author{%
   \name Yi Tian\footnotemark[1] \email{yitian@mit.edu}\\
   \addr{Massachusetts Institute of Technology}\\
   \name Yuanhao Wang\footnotemark[1] \email{yuanhao@princeton.edu}\\
   \addr{Princeton University}\\
   \name Tiancheng Yu\footnotemark[1] \email{yutc@mit.edu}\\
   \addr{Massachusetts Institute of Technology}\\
   \name Suvrit Sra \email{suvrit@mit.edu}\\
\addr{Massachusetts Institute of Technology}\\
}
\renewcommand*{\thefootnote}{\fnsymbol{footnote}}
\footnotetext[1]{Equal contribution.}
\renewcommand*{\thefootnote}{\arabic{footnote}}

\maketitle

\begin{abstract}
    We study online learning in unknown Markov games, a problem that arises in episodic multi-agent reinforcement learning where the actions of the opponents are unobservable. We show that in this challenging setting, achieving sublinear regret against the best response in hindsight is statistically hard. We then consider a weaker notion of regret by competing with the \emph{minimax value} of the game, and present an algorithm that achieves a sublinear $\tilde{\mathcal{O}}(K^{\nicefrac{2}{3}})$ regret after $K$ episodes. This is the first sublinear regret bound (to our knowledge) for online learning in unknown Markov games. Importantly, our regret bound is independent of the size of the opponents' action spaces. As a result, even when the opponents' actions are fully observable, our regret bound improves upon existing analysis (e.g., \citep{xie2020learning}) by an exponential factor in the number of opponents.
\end{abstract}

\section{Introduction}

Multi-agent reinforcement learning (MARL) helps us model strategic decision making problems in an interactive environment with multiple players. It has witnessed notable recent success (with two or more agents), e.g., in Go~\citep{silver2016mastering, silver2017mastering}, video games \citep{vinyals2019grandmaster}, Poker \citep{brown2018superhuman, brown2019superhuman}, and autonomous driving \citep{shalev2016safe}.
 
When studying MARL, often Markov games (MGs) \citep{shapley1953stochastic} are used as the computational model. Compared with Markov decision processes (MDPs) \citep{puterman2014markov}, Markov games allow the players to influence the state transition and returns, and are thus capable of modeling competitive and collaborative behaviors that arise in MARL. 
 
A fundamental problem in MGs is sample efficiency. Unlike MDPs, there are at least two key ways to measure performance in MGs: (1) the offline (self-play) setting, where we control both/all players and aim to minimize the number of episodes required to find a good policy; and (2) the online setting, where we can only control one player (which we refer to as \emph{our player}), treat other players as opponents, and judge how our player performs in the whole process using regret. The offline setting is more useful when training players in a controllable environment (e.g., a simulator) and the online setting is more favorable for life-long learning.
 
When ensuring sample efficiency for MARL, key challenges arise from the observation model. We distinguish between two online settings. When learning in \emph{informed} MGs, our player can observe the actions taken by the opponents. For learning in \emph{unknown} MGs~\citep{cesa2006prediction}, such observations are unavailable; information flows to our player only through the revealed returns and state transitions. We emphasize that both \emph{informed games} and \emph{unknown games} are describing the observation process instead of our prior knowledge of the parameters: We always assume zero knowledge of the transition function of the MG. 
 
Learning in unknown MGs is harder, more general, and potentially of greater practical relevance than informed MGs. It is thus important to discover algorithms that can guarantee low regret. However, theoretical understanding for unknown MGs is rather limited. Even the following \emph{fundamental question} for analyzing online learning in unknown MGs is open:
\begin{center}
    \textbf{Q1.}~~~\emph{Is sublinear regret achievable?}
\end{center}

To see why learning in unknown MGs is challenging, notice that without observing an opponents' actions, we cannot learn the transition function of the MG, even with infinitely many episodes to collect data. Therefore, explore-then-commit type of algorithms cannot achieve sublinear regret.

Another concern arises when the number of players involved increases, as then the effective size of the opponents' action space grows exponentially in it. Therefore, the following question is also crucial, even in (easier) informed MGs:
\begin{center}
  \textbf{Q2.}~~~\emph{Can the regret be independent of the size of the opponents' action space?}
\end{center}

\textbf{Contributions.} We answer both questions Q1 and Q2 affirmatively in this paper. At the heart of our answers lies an Optimistic Nash V-learning algorithm for online learning (\nashvol{}) that we develop. This algorithm is significant in the following aspects:
\begin{list}{$\bullet$}{\leftmargin=1.5em}
    \setlength{\itemsep}{-1pt}
    \item It achieves $\bigotilde(K^{\nicefrac{2}{3}})$ regret, the first sublinear regret bound for online learning in unknown MGs. This bound is nontrivial because without observing opponents' actions, we cannot learn the transition function of the MG, even with infinitely many episodes to collect data.
    \item Its regret does not depend on the size of the opponents' action space. This regret bound is also the first of this kind in the online setting, even for the (easier) \emph{informed} MG setting. For $m$-player MGs, the effective size of the opponents' action space is $A^{m-1}$ with $A$ the size of each player's action space. Therefore, compared with existing algorithms \citep{xie2020learning} even in the informed setting, we save an exponential factor.
    \item It is computationally efficient. The computational complexity does not scale up as the number of players $m$ increases; existing algorithms such as \citep{xie2020learning} suffer space and time complexities exponential in $m$. Also, in existing algorithms, a subprocedure to find a Nash equilibrium in two-player zero-sum games is called in each step, which becomes the computational bottleneck. In sharp contrast, our algorithm does not require calling any such subprocedures.  
\end{list}

The idea of Nash V-learning first appears in \citep{bai2020near}. We denote their original Nash V-learning algorithm by \nashvsp{} (SP is an acronym for self-play) to distinguish it from our algorithm \nashvol{}. See the discussion at the end of Section~\ref{sec:vlearning} for a detailed comparison of the two algorithms.

Furthermore, although the weaker notion of regret (see Section \ref{sec:prelim}) that we use has appeared in prior works \citep{brafman2002r,xie2020learning}, it is not clear why this choice is statistically reasonable. We justify this notion of regret by showing that competing with the best response in hindsight is statistically hard (Section~\ref{sec:lower}). Specifically, the regret can be exponential in the horizon $H$. This result also strengthens the computational lower bound in \citep{bai2020near} for online learning in unknown MGs. As an intermediate step, we prove that competing with the optimal policy in hindsight is also statistically hard in MDPs with adversarial transitions under bandit feedback, which strengthens the computational lower bound in \citep{yadkori2013online} under bandit feedback and is a result of independent interest.
 
\subsection{Related work}

\textbf{Learning in MGs without strategic exploration.} 
A large body of literature focuses on solving known MGs \citep{littman1994markov, hansen2013strategy} or learning with a generative model \citep{jia2019feature, sidford2020solving,zhang2020model}, using which we can sample transitions and returns for arbitrary state-action pairs. \citet{littman2001friend, hu2003nash, wei2017online} do not assume a generative model, but their results only apply to communicating MGs.

\textbf{Online MGs.} 
\citet{brafman2002r} propose R-max, which does not provide a regret guarantee in general. \citet{xie2020learning} study this setting for two-player zero-sum games with linear function approximation using the same weaker definition of regret. They use a value iteration (VI) based algorithm and achieve $\bigotilde(\sqrt{H^4A^3B^3S^3K})$ regret when translated into the tabular language, where $A$ and $B$ are number of actions for the two players, $S$ is the number of states and $H$ is the horizon. In Appendix~\ref{sec:nashq}, we adapt the Optimistic Nash Q-learning algorithm (\nashqsp{})~\citep{bai2020near} to the online setting (\nashqol{}, Algorithm~\ref{alg:nash-q}) and prove for \nashqol{} a $\bigotilde(\sqrt{H^5 A B S K})$ regret (Theorem~\ref{thm:nash-q}). 
All the three algorithms require observing the opponents' actions and thus cannot be applied to learning in unknown MGs. 

\textbf{Self-play.} 
There is a recent line of work focusing on achieving near-optimal sample complexity in offline two-player zero-sum MGs~\citep{bai2020provable, xie2020learning, bai2020near, liu2020sharp}. The goal is to find an $\epsilon$-approximate Nash equilibrium within $K$ episodes. VI-based methods~\citep{bai2020provable, xie2020learning} achieve $K=\bigotilde(S^2AB/\epsilon^2)$. \nashqsp{}~\citep{bai2020near} achieves $K=\bigotilde(SAB/\epsilon^2)$, and the \nashvsp{} algorithm~\citep{bai2020near} achieves the best existing result $K = \bigotilde(S(A+B)/\epsilon^2)$, matching the lower bound w.r.t.\ the dependence on $S$, $A$, $B$ and $\epsilon$. Note that in the self-play setting, we need to find good policies for both players, so the dependence on $B$ is inevitable. Extensions to multi-player general-sum games are discussed in~\citep{liu2020sharp} but the dependence on the number of players is exponential.

\textbf{MDPs with adversarial transitions.} 
Online MGs are closely related to adversarial MDPs. In general, competing with the optimal policy in hindsight in MDPs with adversarial transitions is intractable. With full-information feedback, the problem is computationally hard~\citep{yadkori2013online}. With bandit feedback, the problem is statistically hard (Lemma~\ref{lem:lb-adv-mdp}). However, under additional structural assumptions, one can achieve low regret~\citep{cheung2019non}. 

\textbf{MDPs with adversarial rewards.} 
We can ensure sublinear regret if the transition is fixed (but unknown) and only the reward is chosen adversarially \citep{zimin2013online,rosenberg2019online,jin2019learning}. This yields another useful model for adversarial MDPs. The best existing result in adversarial episodic MDPs with bandit feedback and unknown transition is achieved in \citep{jin2019learning} with $\bigotilde(\sqrt{H^3S^2AK})$ regret, where $H$ is the horizon.

\textbf{Single-agent RL.} 
Finally, there is an abundance of works on sample efficient learning in MDPs. \citet{jaksch2010near} first adopt optimism to achieve efficient exploration in MDPs and \citet{jin2018q} extend this idea to model-free methods. \citet{azar2017minimax} and \citet{zhang2020almost} achieve minimax regret bounds (up to log-factors) $\bigotilde(\sqrt{H^3SAK})$ for model-based and model-free methods, respectively.

\section{Background and problem setup}
\label{sec:prelim}

For simplicity, we formulate the problem of two-player zero-sum MGs in this section and provide our algorithmic solution in Section~\ref{sec:vlearning}. Please see Section~\ref{sec:multi} for extensions to multi-player general-sum MGs.

\subsection{Markov games: setup and notation}
\paragraph{Model.}  We consider episodic two-player zero-sum MGs, where the max-player (min-player) aims to maximize (minimize) its cumulative return. Let $[H] := \{1,2,\ldots,H\}$ for positive integer $H$, and let $\Delta(\cX)$ be the set of probability distribution on set $\cX$.
Then such an MG is denoted by $\mg(\cS, \cA, \cB, \Trans, r, H)$, where 
\begin{list}{–}{\leftmargin=1.5em}
\setlength{\itemsep}{0pt}
    \item $H \in \N_{+}$ is the number of steps in each episode,
    \item $\cS = \bigcup_{h\in [H+1]} \cS_h$ is the state space, 
    \item $\cA = \bigcup_{h\in [H]} \cA_h$ ($\cB = \bigcup_{h\in [H]} \cB_h$) is the action space of the max-player (min player, resp.).
    \item $\Trans$ is a collection of \emph{unknown} transition functions $\{\Trans_h: \cS_h \times \cA_h \times \cB_h \to \Delta(\cS_{h+1})\}_{h\in [H]}$, and
    \item $r$ is a collection of return functions $\{r_h: \cS_h \times \cA_h \times \cB_h \to [0, 1]\}_{h\in [H]}$.
\end{list}
The return $r$ is usually called reward in MDPs, which a player aims to maximize. We will use the term ``return'' for MGs and reserve the term ``reward'' for (adversarial) MDPs. 

With a subscript $h$ let $\cS_h, \cA_h, \cB_h, \Trans_h, r_h$ denote the corresponding objects at step $h$. Let $|\cdot|$ denote cardinality of a set; then define the following terms: 
\begin{align*}
    S := \sup_{h\in [H]} |\cS_h|, \quad 
    A := \sup_{h\in [H]} |\cA_h|, \quad 
    B := \sup_{h\in [H]} |\cB_h|.
\end{align*} 

\paragraph{Interaction protocol.}
In each episode, the MG starts at an adversarially chosen initial state $s_1 \in \cS_1$. At each step $h\in [H]$, the two players observe the state $s_h\in \cS_h$ and simultaneously take actions $a_h\in \cA_h$, $b_h\in \cB_h$; then the environment transitions to the next state $s_{h+1} \sim \Trans_{h}(\cdot\vert s_h, a_h, b_h)$ and outputs the return $r_{h}(s_h, a_h, b_h)$. 
The max-player's policy $\mu$ specifies a distribution on $\cA_h$ at each step $h$. Concretely, $\mu = \{\mu_h\}_{h\in [H]}$ where $\mu_h: \cS_h \to \Delta(\cA_h)$. Similarly we define the min-player's policy $\nu$ .

\paragraph{Value functions.}
Analogously to MDPs, for a policy pair $(\mu, \nu)$, step $h\in [H]$, state $s\in \cS_h$, and actions $a\in \cA_h, b\in \cB_h$, define the state value function and Q-value function as:
\begin{align*}
    V_{h}^{\mu, \nu}(s) &:= \E_{\mu, \nu}[\nlsum_{h'=h}^{H} r_{h'}(s_{h'}, a_{h'}, b_{h'}) \vert s_h = s], \\
    Q_{h}^{\mu, \nu}(s, a, b) &:= \E_{\mu, \nu}[\nlsum_{h'=h}^{H} r_{h'}(s_{h'}, a_{h'}, b_{h'}) \vert s_h = s, a_h = a, b_h = b].
\end{align*}
For compactness of notation, define the operators:
\begin{align*}
    \Trans_h V(s, a, b) := \E_{s'\sim \Trans_h(\cdot\vert s, a, b)}[V(s')], \quad 
    \D_{\mu, \nu}[Q](s) := \E_{a\sim \mu(\cdot\vert s), b\sim \nu(\cdot\vert s)}[Q(s, a, b)].
\end{align*}
Then we have the following Bellman equations:
\begin{align*}
    V_{h}^{\mu, \nu}(s) = \D_{\mu_h, \nu_h} [Q_{h}^{\mu, \nu}](s),\quad 
    Q_{h}^{\mu, \nu}(s, a, b) = (r_h + \Trans_{h}V_{h+1}^{\mu, \nu})(s, a, b).
\end{align*}
For convenience define $V_{H+1}^{\mu, \nu}(s) := 0$ for $s\in \cS_{H+1}$.

\begin{figure*}[t]
  \centering
  \includegraphics[width=0.6\textwidth]{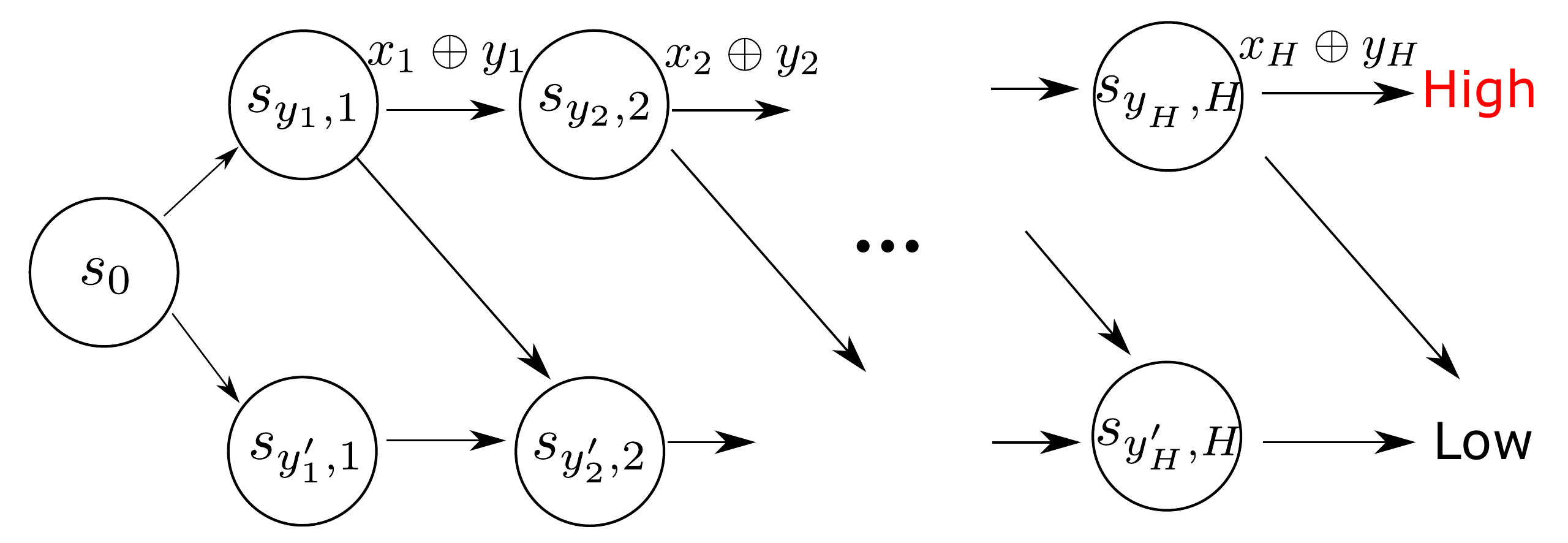}
  \vspace{-0.5cm}
  \caption{Illustration of the MDP $M_{X,Y}$. For $y\in\{0,1\}$, $y'$ stands for $1-y$.}
  \vspace{-0.1cm}
  \label{fig:mdp-short}
\end{figure*}

\paragraph{Optimality.}
For a given min-player's policy $\nu$, there exists a \emph{best response} $\mu^{\dagger} $ to it, such that $V_{h}^{\mu^{\dagger}, \nu}(s) = V_{h}^{\dagger, \nu}(s) := \sup_{\mu} V_{h}^{\mu, \nu}(s)$ for any step $h\in [H]$ and state $s\in \cS_h$. Again, a symmetric discussion applies to the best response to a max-player's policy. The following minimax theorem holds for two-player zero-sum MGs: for any step $h\in [H]$ and state $s\in \cS_h$, 
\begin{align*}
    \max_{\mu} \min_{\nu} V_{h}^{\mu, \nu}(s) = \min_{\nu} \max_{\mu} V_{h}^{\mu, \nu}(s).
\end{align*}
A policy pair $(\mu^{\ast}, \nu^{\ast}) $ that achieves the equality is known as a \emph{Nash equilibrium}. We use $V^{\ast}_h(s):=V_h^{\mu^{\ast},\nu^{\ast}}(s)$ to denote the value at the Nash equilibrium, which is unique for the MG and we call the \emph{minimax value} of the MG.

\subsection{Problem setup}
We are now ready to formally define the problem of online learning in an \emph{unknown} MG: we control the max-player and in each step, only the state $s_h$ and return $r_h$ are revealed, but not the action of the min-player $b_h$. Recall that if $b_h$ is also accessible, we call it the \emph{informed} setting.

Our goal is to maximize the expected cumulative return, or equivalently, to minimize the regret. The conventional definition of regret is to compete against the best fixed policy in hindsight:
\begin{equation}
    \label{eqn:strong-regret}
    \Regret'(K) := \sup_{\mu} \sum_{k=1}^{K} \bigl(  V_{1}^{\mu, \nu^{k}}(s_{1}^{k}) - V_{1}^{\mu^{k}, \nu^{k}}(s_{1}^{k}) \bigr),
\end{equation}
where the superscript $k$ denotes the corresponding objects in the $k$th episode. Although we use this compact notation, the regret depends on both $\mu^k$ and $\nu^k$.

However, even in the informed setting, achieving sublinear regret in this form is computationally hard~\citep{bai2020near}. For online learning in unknown MGs, the problem is statistically hard (Section~\ref{sec:lower}), thus is still intractable even if we have infinite computational power. 

Therefore, by noting 
\begin{align*}
    \max_{\mu\in M} V_{1}^{\mu, \nu^{k}}(s_{1}^{k}) \ge V_{1}^{\mu^{\ast}, \nu^{k}}(s_{1}^{k}) \ge V_{1}^{\ast}(s_{1}^{k}),
\end{align*}
we consider a more modest goal. That is, to minimize the following regret against the minimax value of the game, which has appeared in prior works \citep{brafman2002r, xie2020learning}: 
\begin{align}
    \label{eqn:weak-regret}
    \Regret(K) := \nlsum_{k=1}^{K} \bigl( V_{1}^{\ast}(s_1^{k}) - V_{1}^{\mu^{k}, \nu^{k}}(s_{1}^{k}) \bigr).
\end{align}

\section{Statistical hardness of online learning in unknown MGs}
\label{sec:lower}

As mentioned above, we use the minimax value of the game as the benchmark for online learning in unknown MGs. In contrast, in adversarial MDPs~\citep{jin2019learning}, it is more common to compete against the best policy in hindsight (using regret~\eqref{eqn:strong-regret}). In this section, we justify our usage of the weaker notion of regret \eqref{eqn:weak-regret} by showing that, in general, competing against the best policy in hindsight is statistically intractable. In particular, we show that in this case, the regret has to be either linear in $K$ or exponential in $H$. 

\begin{theorem}[Statistical hardness for online learning in unknown MGs]
  \label{theorem:lower-bound-uncoupled}
  For any $H\ge 2$ and $K\ge 1$, there exists a two-player zero-sum MG with horizon $H$, $|S_h|\le 2$, $|A_h|\le 2$, $|B_h|\le 4$ such that any algorithm for unknown MGs suffers the following worst-case one-sided regret:
  \begin{align*}
    \sup_{\mu} \nlsum_{k=1}^K\Bigl(V_1^{\mu,\nu^k}(s_1)-\E_{\mu^k}V_1^{\mu^k,\nu^k}(s_1)\Bigr)
    \ge \Omega\bigl(\min\bigl\{\sqrt{2^HK}, K\bigr\}\bigr).
  \end{align*}
  In particular, any algorithm has to suffer linear regret unless $K\ge \Omega(2^H)$.
\end{theorem}

Here we give a sketch of our proof, while the full proof is deferred to Appendix \ref{appendix:proof-lower-bound-uncoupled}.

We start by considering online learning in (single-agent) MDPs, where the reward and transition function in each episode are adversarially determined, and the goal is to compete against the best (fixed) policy in hindsight. In the following lemma we show that this problem is statistically hard; see Lemma~\ref{lem:lb-adv-mdp} in the appendix for its formal statement.

\begin{customlem}{(informal)}
    \label{lem:lb-adv-mdp-infml}
    For any algorithm, there exists a sequence of single agent MDPs with horizon $H$, $S=O(H)$ states and $A=O(1)$ actions, such that the regret defined against the best policy in hindsight is $\Omega(\min\{\sqrt{2^HK}, K\})$.
\end{customlem}

\begin{rmk}
    The above lemma is different from a previous hardness result in~\citet{yadkori2013online}, which states that this problem is \emph{computationally} hard.
\end{rmk}

We now briefly explain how this family of hard MDPs is constructed, which is inspired by the ``combination lock'' MDP~\citep{du2019provably}. Every MDP $M_{X,Y}$ is specified by two $H$-bit strings: $X,Y\in\{0,1\}^H$. The states are $\{s_{0,0},s_{0,1},s_{1,1},\cdots,s_{0,H},s_{1,H}\}$. As shown in Figure~\ref{fig:mdp-short}, $M_{X,Y}$ has a layered structure, and the reward is nonzero only at the final layer. The only way to achieve the high reward is to follow the path $s_{0,0}\to s_{y_1,1}\to \cdots s_{y_H,H}$. Thus, the corresponding optimal policy is $\pi(s_{w,h})=x_h\oplus w$, which is only a function of $X$. Here, $\oplus$ denotes the bitwise exclusive or operator.

Now, in each episode, $Y$ is chosen from a uniform distribution over $\{0,1\}^H$ while $X$ is fixed. When the player interacts with $M_{X,Y}$, since $Y$ is uniformly random, it gets no effective feedback from the observed transitions, and the only informative feedback is the reward at the end. However, achieving the high reward requires guessing every bit of $X$ correctly. This ``needle in a haystack'' situation makes the problem as hard as a multi-armed bandit problem with $2^H$ arms. The regret lower bound immediately follows.

Next, we use the hard family of MDPs in Lemma~\ref{lem:lb-adv-mdp} to prove Theorem~\ref{theorem:lower-bound-uncoupled} by reducing the adversarial MDP problem to online learning in unknown MGs. The construction is straightforward. The state space and the action space for the max-player are the same as that in the original MDP family. The min-player has control over the transition function and reward at each step, and executes a policy such that the induced MDP for the max-player is the same as $M_{X,Y}$. This is possible using only $B=O(1)$ actions as $M_{X,Y}$ has a layered structure. Online learning in unknown MGs then simulates the online learning in the adversarial MDP problem, and thus has the same regret lower bound.

\textbf{Classes of policies.} 
In Section~\ref{sec:prelim}, we define the policy $\mu$ by mappings from $\cS_h$ to a distribution on $\cA_h$ at each step $h$. Such policies are called \emph{Markov policies} \citep{bai2020near}. The policies induced by the algorithms in the remaining part of this paper are always Markov policies. However, our lower bound also holds for \emph{general policies}~\citep{bai2020near}. Here, for an informed max-player the input of $\mu_{h}$ can be the history $(s_1, a_1, b_1, r_1, \cdots, s_h)$, while for a max-player in an unknown MG the input of $\mu_h$ can be the history $(s_1, a_1, r_1, \cdots, s_h)$. In words, the lower bound holds even for policies that depend on histories.

\textbf{Regret minimization in self-play.} 
We emphasize that our lower bound applies to online learning in unknown MGs. For the self-play setting, people indeed minimize the strong regret~\eqref{eqn:strong-regret} as an intermediate step toward PAC guarantees~\citep{bai2020provable, bai2020near, xie2020learning}. This is possible because in self-play \emph{both} players are running the policies specified by the algorithm designer. Therefore, they do not need to worry about the adversarial scenario described in the lower bound here.

\section{The \nashvol{} algorithm}
\label{sec:vlearning}

\begin{algorithm}[t]
    \caption{Optimistic Nash V-learning for Online Learning (\nashvol{})}
    \label{alg:nash-v}
    \begin{algorithmic}[1]
        \State {\bfseries Require:} Learning rate $\{\alpha_t\}_{t\ge 1}$, exploration bonus $\{\beta_t\}_{t\ge 1}$, policy update parameter $\{\eta_t\}_{t\ge 1}$
        \State {\bfseries Initialize:} for any $h\in [H], s\in \cS_h, a\in \cA_h$,
        $\uV_{h}(s)\setto H$, $\uL_{h}(s, a)\setto 0$, $N_{h}(s)\setto 0$, $\mu_h(a\vert s) \setto 1/|\cA_h|$.
        \For{episode $k=1,\dots,K$}
            \State Receive $s_1$
            \For{step $h=1,\dots, H$}
                \State Take action $a_h \sim \mu_h(\cdot| s_h)$
                \State Observe return $r_h$ and next state $s_{h+1}$
                \State Increase counter $t = N_{h}(s_h)\setto N_{h}(s_h) + 1$
                \State $\uV_h(s_h) \setto (1-\alpha_t)\uV_h(s_h) + \alpha_t(r_h + \uV_{h+1}(s_{h+1}) + \ubeta_t)$ 
                \For{all actions $a\in \cA_h$}
                    \State $\ul_h(s_h, a) \setto (H - r_h - \uV_{h+1}(s_{h+1})) \I(a_h=a) / (\mu_h(a_h \vert s_{h}) + \ueta_t)$
                    \State $\uL_h(s_h, a) \setto (1 - \alpha_t ) \uL_h(s_h, a) + \alpha_t \ul_h(s_h, a)$
                \EndFor
                \State Update policy $\mu$ by 
                \begin{align*}
                    \mu_h(\cdot \vert s_h) \setto \frac{\exp\{ -\ueta_t \uL_{h}(s_h, \cdot) / \alpha_t \}}{\sum_{a} \exp\{ -\ueta_t \uL_{h}(s_h, a) / \alpha_t \}}   
                \end{align*}
            \EndFor
        \EndFor
    \end{algorithmic}
\end{algorithm}

In this section, we introduce the \nashvol{} algorithm and its regret guarantees for online  learning in two-player zero-sum \emph{unknown} Markov games. We show that not only can we achieve a sublinear regret in this challenging setting, but the regret bound can be independent of the size of the opponent's action space as well.

\paragraph{The \nashvol{} algorithm.}
\nashvol{} is a variant of V-learning algorithms. \citet{bai2020near} first propose \nashvsp{} as a near-optimal algorithm for the self-play setting of two-player zero-sum MGs. See the discussion at the end of this section for a detailed comparison between \nashvol{} and \nashvsp{}.

In \nashvol{} (Algorithm~\ref{alg:nash-v}), at each time step $h$, the player interacts with the environment, performs an incremental update to $\uV_h$, and updates its policy $\mu_h$. Note that the estimated value function $\uV_h$ is only used for the intermediate loss $l_h(s_h,\cdot)$ in this time step, but not used in decision making. To encourage exploration in less visited states, we add a bonus term $\beta_t$. As we will see in Section~\ref{sec:proof-sketch}, this update rule is optimistic, i.e., $\uV_h$ is an upper confidence bound (UCB) on the minimax value $\oV{h}$ of the MG. Then the player samples the action according to the exponentially weighted averaged loss $L_h(s_h,\cdot)$, which is a popular decision rule in adversarial environments~\citep{auer1995gambling}.

\paragraph{Intuition behind V-learning.} 
Most existing provably efficient tabular RL algorithms learn a Q-table (table consisting of Q-values). However, since state-action pairs are necessary for updating the Q-table, for online learning in MGs, algorithms based on it inevitably require observing the opponent's actions and are thus inapplicable to unknown MGs. In contrast, \nashvol{} does not need to maintain the Q-table at all and bypasses this challenge naturally. 

Moreover, learning a Q-value function in two-player Markov games usually results in a regret or sample complexity that depends on its size $SAB$, whether in the self-play setting, such as VI-ULCB~\citep{bai2020provable} and \nashqsp{}~\citep{bai2020near}, OMNI-VI-offline~\citep{xie2020learning}, or in the online setting, such as OMNI-VI-online~\citep{xie2020learning} and \nashqol{} (Appendix~\ref{sec:nashq}). In contrast, V-learning is promising in removing the dependence on $B$, as formalized in Theorem~\ref{thm:nash-v}. 

Note that we analyze \nashqol{} in Appendix~\ref{sec:nashq} to more clearly demonstrate \nashvol{}'s advantage of avoiding learning a Q-table. \nashqol{} is a Q-learning-type algorithm for online MGs adapted from \nashqsp{}. It updates the Q-values by a termporal difference method like \nashvol{} but makes decisions based on the Q-values instead. Therefore, \nashqol{} applies only to the informed setting and its regret depends on $AB$ (Theorem~\ref{thm:nash-q}).

\paragraph{Favoring more recent samples.} 
Despite the above noted advantages of V-learning, the \nashvsp{} algorithm~\citep{bai2020near} may have a regret bound that is linear in $K$, as indicated by~\eqref{eqn:nash-v-regret} in Theorem~\ref{thm:nash-v} and discussed in Section~\ref{sec:proof-sketch} in more detail. To resolve this problem, we adopt a different set of hyperparameters to learn more aggressively by giving more weight to more recent samples. 
Concretely, for the self-play setting, \citet{bai2020near} specify the following hyperparameters for \nashvsp{}: 
\begin{align*}
    \alpha_t = \tfrac{H + 1}{H + t}, \;
    \ubeta_t = c\sqrt{\tfrac{H^4 A \iota}{t}}, \;
    \ueta_t = \sqrt{\tfrac{\log A}{At}},
\end{align*}
where $\iota$ is a log factor defined later.
For the online setting, we set these hyperparameters as:  
\begin{align}  \label{eqn:alpha-beta-eta}
    \alpha_{t} = \tfrac{GH + 1}{GH + t}, \;
    \ubeta_t = c  \sqrt{\tfrac{GH^3 A\iota}{t}}, \;
    \ueta_t = \sqrt{\tfrac{GH \log A}{At}}, 
\end{align}
where $G\ge 1$ is a quantity that we tune. Ostensibly, these changes may appear small, but they are essential to attaining a sublinear regret.
\begin{rmk}
    Compared with $\alpha_t = \nicefrac{1}{t}$, the learning rate $\alpha_t = \nicefrac{H+1}{H+t}$ first proposed in~\citep{jin2018q} already favors more recent samples. Here we go one step further: our algorithm learns even more aggressively by taking $\alpha_t = \nicefrac{GH+1}{GH+t}$ with $G\ge 1$. Moreover, we choose a larger $\eta_t$ to make our algorithm care more about more recently incurred loss. $\beta_t$ is set accordingly to achieve optimism.
\end{rmk}
We call this variant of V-learning \nashvol{}, for which we prove the following regret guarantees. 
\begin{theorem}[Regret bounds]  \label{thm:nash-v}
    For any $p \in (0, 1)$, let $\iota = \log (\nicefrac{HSAK}{p})$. If we run \nashvol{} with our hyperparameter specification~\eqref{eqn:alpha-beta-eta} for some large constant $c$ and $G \ge 1$ in an online two-player zero-sum MG, then with probability at least $1 - p$, the regret in $K$ episodes satisfies
    \begin{align}  \label{eqn:nash-v-regret}
        \Regret(K) 
        = \bigo\bigl( \sqrt{GH^5 SAK\iota} + KH/G + H^2 S \bigr).
    \end{align}
    In particular, by taking $G = H^{-1}(K/SA)^{\nicefrac{1}{3}}$ if $K \ge H^3 SA$ and $G = K^{\nicefrac{1}{3}}$ otherwise, with probability at least $1 - p$, the regret satisfies 
    \begin{align*}
        \Regret(K) = 
        \begin{cases}
            \bigotilde\bigl( H^2 S^{\frac{1}{3}} A^{\frac{1}{3}} K^{\frac{2}{3}} + H^2 S \bigr), \text{ if } K \ge H^3SA, \\
            \bigotilde\bigl( \sqrt{H^5 SA} K^{\frac{2}{3}} + H^2 S \bigr), \text{ otherwise.} \\
        \end{cases}
    \end{align*}
\end{theorem}

Theorem~\ref{thm:nash-v} shows that a sublinear regret against the minimax value of the MG is achievable for online learning in unknown MGs. 
As expected, the regret bound does not depend on the size of the opponent's action space $B$. This independence of $B$ is particularly significant for large $B$, as is the case where our player plays with multiple opponents.
Note that although in Theorem~\ref{thm:nash-v} setting the parameter $G$ requires knowledge of $K$ beforehand, we can use a standard doubling trick to bypass this requirement.

\begin{rmk}
    In \nashvsp{} the parameter $G$ is set to be $1$. Then our choice of $\eta_t$ becomes $\sqrt{\nicefrac{H \log A}{A t}}$, $\sqrt{H}$ times the original policy update parameter. If the other player also adopts the new $\sqrt{\nicefrac{H \log B}{B t}}$ policy update parameter, then the sample complexity of \nashvsp{} can actually be improved upon~\citep{bai2020near} by an $\sqrt{H}$ factor to $\bigotilde(H^5S(A+B) / \epsilon^2)$.
\end{rmk}

\paragraph{Comparison between \nashvol{} and \nashvsp{}.} 
\begin{enumerate}  
\setlength\itemsep{0em}
    \item To achieve near-optimal sample complexity in the self-play setting, \nashvsp{} needs to construct upper and lower confidence bounds not only for the \emph{minimax value} of the game, but also for the \emph{best response}. As a result, it uses a complicated certified policy technique, and it must store the whole history in the past $K$ episodes for resampling from in each step. By comparing with the \emph{minimax value} directly, we can make \nashvol{} provably efficient without extracting a certified policy. Therefore, \nashvol{} only needs $\cO(SAH)$ space instead of $\cO(KSAH)$, and the resampling procedure is no more necessary. 
    \item A key feature of the proof in \citep{bai2020near} is to make full use of a symmetric structure, which naturally arises because in the self-play setting we can control both players to follow the same learning algorithm. However, this property no longer holds for the online setting, and we must take a different proof route. Algorithmically, we need to learn more aggressively to make \nashvol{} provably efficient.
    \item \nashvol{} also works in multi-player general-sum MGs---see Section~\ref{sec:multi}.
\end{enumerate}

\section{Multi-player general-sum games}
\label{sec:multi}

In this section, we extend the regret guarantees of \nashvol{} to multi-player general-sum MGs, demonstrating the generality of our algorithm. Informally, we have the following corollary.

\begin{customcor}{(informal)} 
    \label{cor:multi-mg}
    If we run \nashvol{} with our hyperparameter specificified in~\eqref{eqn:alpha-beta-eta} for our player in an online multi-player general-sum MG, then with high probability, for sufficiently large $K$,
    \begin{align*}
        \Regret(K) 
        = \bigotilde\bigl( H^2 S^{\frac{1}{3}} A^{\frac{1}{3}} K^{\frac{2}{3}} + H^2 S \bigr),
    \end{align*}
    where $A$ denotes the size of our player's action space.
\end{customcor}

The above corollary highlights the significance of removing the dependence on $B$ in the regret bound. In particular, in a multi-player game the size of the opponents' joint action space $B$ grows exponentially in the number of opponents, whereas the regret of \nashvol{} only depends on the size of our player's action space $A$. The savings arise because \nashvol{} bypasses the need to learn Q-tables, and the multi-player setting makes no real difference in our analysis. To formally present the construction, we need to first introduce some notation.

Consider the $m$-player general-sum MG 
\begin{align}  \label{eqn:mp-gs-mg}
    \mg_m(\cS, \{\cA_i\}_{i=1}^m, \Trans, \{r_{i}\}_{i=1}^{m}, H),     
\end{align}
where $\cS$, $H$ follow from the same definition in two-player zero-sum MGs, and 
\begin{list}{–}{\leftmargin=1.5em}
    \setlength\itemsep{0em}
    \item for each $i\in [m]$, player $i$ has its own action space $\cA_i = \bigcup_{h\in [H]} \cA_{i, h}$ and return function $r_i = \{r_{i, h}: \cS_h \times \bigotimes_{i=1}^{m} \cA_{i, h} \to [0, 1] \}_{i=1}^{m}$, and aims to maximize its own cumulative return (here $\bigotimes$ denotes the Cartesian product of sets);
    \item $\Trans$ is a collection of transition functions $\{\Trans_h: \cS_h \times \bigotimes_{i=1}^{m} \cA_{i, h} \to \Delta(\cS_{h+1})\}_{h\in [H]}$.
\end{list}

Like in two-player MGs, let 
\begin{align*}
    S := \sup_{h\in [H]} |\cS_h|, \quad 
    A_i := \sup_{h\in [H]} |\cA_i, h| \text{ for all } i \in [m]. 
\end{align*} 
Online learning in an unknown multi-player general-sum MG can be reduced to that in a two-player zero-sum MG.
Concretely, suppose we are player $1$, then online learning in unknown MGs \eqref{eqn:mp-gs-mg} is indistinguishable from that in the two-player zero-sum MG specified by $(\cS, \cA_1, \cB, \Trans, r_1, H)$ where $\cB = \bigotimes_{i=2}^{m} \cA_i$, since we only observe and care about player $1$'s return. 
For all states $s\in \cS_1$, define the value function using $r_1$ as 
\begin{align*}
    V_{h}^{\mu, \nu}(s) := \E_{\mu, \nu}[\nlsum_{h'=h}^{H} r_{1, h'}(s_{h'}, a_{h'}, b_{h'}) \vert s_h = s], 
\end{align*}
and define the minimax value of player $1$ as 
\begin{align*}
    \oV{1}(s) := \max_{\mu} \min_{\nu} V_{1}^{\mu, \nu}(s) = \min_{\nu} \max_{\mu} V_{1}^{\mu, \nu}(s), 
\end{align*}
which is no larger than the value at the Nash equilibrium of the multi-player general-sum MG. 
Then we define the regret against the minimax value of player $1$ as 
\begin{align*}
    \Regret(K) := \nlsum_{k=1}^{K} \bigl( \oV{1}(s_1^{k}) - V_{1}^{\mu^{k}, \nu^{k}}(s_{1}^{k}) \bigr).
\end{align*}
We argue that this notion of regret is reasonable since we have control of only player $1$ and all opponents may collude to compromise our performance.
Then immediately we obtain the following corollary from Theorem~\ref{thm:nash-v}.

\begin{corollary}[Regret bound in multi-player MGs]
    For any $p \in (0, 1)$, let $\iota = \log (\nicefrac{HSAK}{p})$. If we run \nashvol{} with our hyperparameter specification~\eqref{eqn:alpha-beta-eta} for some large constant $c$ and the above choice of $G$ for player 1 in the online multi-player general-sum MG~\eqref{eqn:mp-gs-mg}, then with probability at least $1 - p$, the regret in $K$ episodes satisfies
    \begin{align*}
        \Regret(K) = 
        \begin{cases}
            \bigotilde\bigl( H^2 S^{\frac{1}{3}} A_1^{\frac{1}{3}} K^{\frac{2}{3}} + H^2 S \bigr), \text{ if } K \ge H^3SA_1, \\
            \bigotilde\bigl( \sqrt{H^5 SA_1} K^{\frac{2}{3}} + H^2 S \bigr), \text{ otherwise.} \\
        \end{cases}
    \end{align*}
\end{corollary}
In the online informed setting, the same equivalence to a two-player zero-sum MG holds, since the other players' actions we observe can be seen as a single action $(a_i)_{i=2}^{m}$, and whether we observe the other players' returns does not help us decide our policies to maximize our own cumulative return. In this setting, the regret bound in \citep{xie2020learning} becomes $\bigotilde(\sqrt{H^3 S^3 \prod_{i=1}^m A_i^3 T})$, which depends exponentially on $m$. On the other hand, since the online informed setting has stronger assumptions than online learning in unknown MGs, the $\bigotilde(H^2 S^{\nicefrac{1}{3}} A_1^{\nicefrac{1}{3}} K^{\nicefrac{2}{3}})$ regret bound of \nashvol{} carries over, which has no dependence on $m$. This sharp contrast highlights the importance of achieving a regret independent of the size of the opponent's action space. 

Furthermore, since in \nashvol{} we only need to update the value function (which has $H S$ entries), rather than update the Q-table (which has $H S\prod_{i=1}^m A_i$ entries) as in~\citep{xie2020learning}, we can also improve the time and space complexity by an exponential factor in $m$. 

\section{Proof sketch of Theorem~\ref{thm:nash-v}}
\label{sec:proof-sketch}

In this section, we sketch the proof of Theorem~\ref{thm:nash-v}. We also highlight an observation that \nashvol{} can perform much better than claimed in Theorem \ref{thm:nash-v}. Moreover, we expose the problem with \nashvsp{} in the online setting, which explains why we favor more recent samples in \nashvol{}.

In the analysis below, we use a superscript $k$ to signify the corresponding quantities at the beginning of the $k$th episode. To express $\uV_{h}^{k}$ in Algorithm~\ref{alg:nash-v} compactly, we introduce the following quantities. 
\begin{align*}
    \alpha_t^0 := \nlprod_{j=1}^{t} (1 - \alpha_j), \quad \alpha_t^i := \alpha_i \nlprod_{j=i+1}^{t} (1 - \alpha_j).
\end{align*}
Let $t := N_{h}^{k}(s)$ and suppose $s$ is previously visited at episodes $k^1,\ldots, k^t \le k$. Then we can express $\uV_{h}^{k}(s)$ as 
\begin{align*}
    \alpha_{t}^{0} H + \nlsum_{i=1}^{t} \alpha_{t}^{i} \bigl( r_{h}(s, a_{h}^{k^i}, b_{h}^{k^i}) + \uV_{h+1}^{k^i}(s_{h+1}^{k^i}) + \ubeta_{i} \bigr).
\end{align*}
It is easy to verify that $\{\alpha_t^i\}_{i=1}^t$ satisfies the normalization property that $\sum_{i=1}^{t} \alpha_{t}^{i} = 1$ for any sequence $\{\alpha_t\}_{t\ge 1}$ and any $t\ge 1$.
Moreover, for $\{\alpha_t\}_{t\ge 1}$ specified in~\eqref{eqn:alpha-beta-eta}, $\{\alpha_t^i\}$  has several other desirable properties (Lemma~\ref{lem:ppt-alpha}), resembling~\citep[Lemma 4.1]{jin2018q}.

\paragraph{Upper confidence bound (UCB).}
In Algorithm~\ref{alg:nash-v}, by bonus $\ubeta_t$ we ensure that $\uV_{h}^{k}$ is an entrywise UCB on $\oV{h}$ using standard techniques~\citep{bai2020near}, building on the normalization property of $\{\alpha_t^i\}_{i=1}^t$ and the key V-learning lemma (Lemma~\ref{lem:v-est-err}) based on the regret bound of the adversarial bandit problem we solve to derive the policy update.

\begin{rmk}
    \begingroup
    \allowdisplaybreaks
    A main difference from the previous UCB framework (e.g., \citet{azar2017minimax}) is that here the gap between $\uV_h^{k}$ and $\oV{h}$ is not necessarily diminishing, which partially explains why we do not achieve the conventional $\bigotilde(\sqrt{T})$ regret. Concretely, by taking $\mu = \mu^{\ast}$ in the V-learning lemma (Lemma~\ref{lem:v-est-err}), we have 
    \begin{align*}
        &\uV_{h}^{k}(s) - \oV{h}(s) \\
        \ge & \nlsum_{i=1}^{t} \alpha_{t}^{i} \D_{\mu_{h}^{\ast}, \nu_{h}^{k^i}} [r_{h} + \Trans_{h} \uV_{h+1}^{k^i}](s) - \D_{\mu_{h}^{\ast}, \nu_{h}^{\ast}}[r_{h} + \Trans_{h} \oV{h+1}](s) \\
        = & \nlsum_{i=1}^{t} \alpha_{t}^{i} \D_{\mu_{h}^{\ast}, \nu_{h}^{k^i}} [\Trans_{h} (\uV_{h+1}^{k^i} - \oV{h+1})](s) + \nlsum_{i=1}^{t}\alpha_{t}^i (\D_{\mu_{h}^{\ast}, \nu_h^{k^i}} - \D_{\mu_{h}^{\ast}, \nu_{h}^{\ast}})[r_{h} + \Trans_{h} \oV{h+1}](s) \\
        \overset{(i)}{\ge} & \nlsum_{i=1}^{t}\alpha_{t}^i (\D_{\mu_{h}^{\ast}, \nu_h^{k^i}} - \D_{\mu_{h}^{\ast}, \nu_{h}^{\ast}})[r_{h} + \Trans_{h} \oV{h+1}](s),
    \end{align*}
    \endgroup
    where $(i)$ follows from the above UCB.
    If the opponent is weak at some step $h\in [H]$ such that for all episodes $k\in [K]$, 
    \begin{align*}
        (\D_{\mu_{h}^{\ast}, \nu_h^{k}} - \D_{\mu_{h}^{\ast}, \nu_{h}^{\ast}})[r_{h} + \Trans_{h} \oV{h+1}](s) \ge C,  
    \end{align*}
    then $\sum_{k=1}^{K} (\uV_{h}^{k}(s) - \oV{h}(s)) \ge CK$. 
    This indicates that the gap between the sum of the UCBs and that of the minimax values can be linear in $K$. As proved below, we actually show that $\nlsum_{k=1}^{K} (\uV_{1}^{k} - V_{1}^{\mu^{k}, \nu^{k}})(s_{h}^{k})$ is sublinear in $K$, which is much stronger than that merely the regret is sublinear if the opponent is weak. In words, \nashvol{} performs much better than claimed in Theorem \ref{thm:nash-v} against a weak opponent.
\end{rmk}

\paragraph{Regret bounds.}
Note that the above proof of the UCB holds for any $G > 0$. We now illustrate what problem appears if $G = 1$ and where the constraint $G\ge 1$ comes from. 
Let ``$\lesssim$'' denote ``$\le$'' up to multiplicative constants and log factors. By definition, we have 
\begin{align*}
    \delta_{h}^{k} 
    := (\uV_{h}^{k} - V_{h}^{\mu^{k}, \nu^{k}})(s_{h}^{k})
    \lesssim \sqrt{\tfrac{GH^3 A\iota}{t}} - \D_{\mu_{h}^{k}, \nu_{h}^{k}} [r_{h} + \Trans_{h} V_{h+1}^{\mu^{k}, \nu^{k}}](s_{h}^{k}) 
    + \nlsum_{i=1}^{t} \alpha_{t}^{i} \D_{\mu^{k^i}, \nu^{k^i}} [r_h + \Trans_{h} \uV_{h+1}^{k^i}](s_{h}^{k}).
\end{align*}
By the same regrouping technique as that in \citet{jin2018q}, for any quantity $f^i$ indexed by $i\in [K]$, 
\begin{align*}
    \nlsum_{k=1}^{K} \nlsum_{i=1}^{t} \alpha_{t}^{i} f^{k^i}
    \le \nlsum_{k'=1}^{K} f^{k'} \nlsum_{t=n_{h}^{k'}}^{\infty} \alpha_{t}^{n_{h}^{k'}}
    \le (1 + \tfrac{1}{GH}) \nlsum_{k=1}^{K} f^{k}.
\end{align*}
Taking $\D_{\mu^{k^i}, \nu^{k^i}} [r_h + \Trans_{h} \uV_{h+1}^{k^i}](s_{h}^{k})$ as $f^i$ and substituting the resulting bound into $\sum_{k=1}^{K} \delta_h^k$ yields
\begin{align*}
    \nlsum_{k=1}^{K} \delta_{h}^{k}
    & \lesssim \nlsum_{k=1}^{K} \bigl( \sqrt{\tfrac{GH^3 A\iota}{t}} + (1 + \tfrac{1}{GH}) \delta_{h+1}^{k} 
    + \tfrac{1}{GH} \D_{\mu_{h}^{k}, \nu_{h}^{k}}[r_{h} + \Trans_{h} V_{h+1}^{\mu^{k}, \nu^{k}}](s_{h}^{k}) \bigr) \\
    & \overset{(i)}{\le} \nlsum_{k=1}^{K} \bigl( \sqrt{\tfrac{GH^3 A\iota}{t}} + (1 + \tfrac{1}{GH}) \delta_{h+1}^{k} + \tfrac{1}{G} \bigr),
\end{align*}
where $(i)$ is owing to $\D_{\mu_{h}^{k}, \nu_{h}^{k}}[r_{h} + \Trans_{h} V_{h+1}^{\mu^{k}, \nu^{k}}](s_{h}^{k}) \le H$.
Since $\sum_{k=1}^{K} \delta_{H+1}^k = 0$, a recursion over $h\in [H]$ for $\sum_{k=1}^{K} \delta_{h}^{k}$ yields 
\begin{align*}
    \nlsum_{k=1}^{K} \delta_{1}^{k} 
    \lesssim (1 + \tfrac{1}{GH})^{H} \bigl( \sqrt{\tfrac{GH^3 A\iota}{t}} + \tfrac{1}{G} \bigr).
\end{align*}
To bound the coefficient $(1 + \tfrac{1}{GH})^{H} \le e$, we need $G \ge 1$.
By noting 
\begin{align*}
    \nlsum_{k=1}^{K} \sqrt{\tfrac{1}{t}} 
    = \nlsum_{k=1}^{K} \sqrt{\tfrac{1}{n_{h}^{k}}} 
    \le \sqrt{S K},
\end{align*}
we obtain 
\begin{align*}
    \Regret(K) \le \sum_{k=1}^{K} \delta_{1}^{k} 
    \lesssim \sqrt{GH^5 SAK\iota} + G^{-1} KH.
\end{align*}
If we take $G = 1$ as in \nashvsp{}, the regret is linear in $K$ and therefore useless. To address this problem, we introduced the tunable parameter $G\ge 1$ that balances the $\sqrt{K}$ and $K$ terms in the above bound to yield a sublinear regret.

\section{Conclusion and Future Work}
In this paper, we study online learning in unknown Markov games using \nashvol{}, which is based on the \nashvsp{} algorithm of~\citet{bai2020near}. \nashvol{} achieves $\bigotilde(K^{\nicefrac{2}{3}})$ regret after $K$ episodes. Furthermore, the regret bound is independent of the size of opponents' action space. It is still unclear whether one can achieve a sharper regret bound, which is a question worthy of future study. We briefly comment on two other future directions.

\textbf{Toward $\bigotilde(K^{\nicefrac{1}{2}})$ regret in MDPs.} 
A key reason why we need to learn more aggressively in online learning is that a symmetric structure (like in the proof of \nashvsp{}) is absent. However, it exists if the opponent plays a \emph{fixed} policy, in which case the Markov game becomes an MDP. To see why, we can imagine the opponent is also executing \nashvol{}, which makes no difference since $B=1$. However, even in that case, a gap remains: we can only upper and lower bound $V^{*}$ but not $V_{h}^{\mu^{k}, \nu^{k}}$. Figuring out how to fill this gap will make \nashvol{} become the first policy-based algorithm without an estimation of Q-value functions that achieves a $\bigotilde(K^{\nicefrac{1}{2}})$ regret for tabular RL.

\textbf{Strong regret for MDPs with adversarial rewards.} Another special case is MDPs with adversarial rewards, where the transitions are fixed across episodes. In this case, achieving sublinear regret using strong regret~\eqref{eqn:strong-regret} is possible~\citep{jin2019learning}. A question is then: does \nashvol{} (or its variants) achieve sublinear regret using the strong regret? Given the many technical differences between adversarial MDPs and online Markov games, it is desirable to resolve these problems in a unified manner. In addition, the form of the model-free update in \nashvol{} should be of independent interest for MDPs with adversarial rewards.

\section*{Acknowledgement}
YT, TY, SS acknowledge partial support from the NSF BIGDATA grant (number 1741341). We thank Yu Bai, Kefan Dong and Chi Jin for useful discussions.

\bibliographystyle{plainnat}
\bibliography{online_mg}
\clearpage

\appendix

\section{Proof of the lower bound}
\label{appendix:proof-lower-bound-uncoupled}

The lower bound builds on the following lower bound for adversarial MDPs where both the transition and the reward function of each episode are chosen adversarially. Note that in our proof of Lemma~\ref{lem:lb-adv-mdp}, the optimal policies for $M_k$ are the same, so Lemma~\ref{lem:lb-adv-mdp} indeed implies a lower bound on the regret defined against the best stationary policy in hindsight.
\begin{lemma}[Lower bound for adversarial MDPs]
    \label{lem:lb-adv-mdp}
    For any horizon $H\ge 2$ and $K\ge 1$, there exists a family of MDPs $\cM$ with horizon $H$, state space $\{S_h\}_{h\le H}$ with $|S_h|\le 2$, action space $\set{A_h}_{h\le H}$ with $|A_h|\le 2$, and reward $r_h\in [0,1]$ such that the following is true: for any algorithm that deploys policy $\mu^k$ in episode $k$, we have
    \begin{equation*}
        \sup_{M_1,\cdots,M_K\in \cM} \sup_{\mu} \sum_{k=1}^K \left( V^\mu_{M_k}(s_0) - \E_{\mu_k} V^{\mu^k}_{M_k}(s_0) \right) \ge \Omega(\min\set{\sqrt{2^HK}, K}),
    \end{equation*}
    where $V^\ast_{M_k}$ refers to the optimal value function of MDP $M_k$.
\end{lemma}
\begin{proof}
    Our construction is inspired by the ``combination lock'' MDP~\citep{du2019provably}. Let us redefine the horizon length as $H+1$ (so that $H\ge 1$) and let $h$ start from 0. We now define our family of MDPs.
  
    \begin{definition}[MDP $M_{X,Y,\eps}$]
        For any pair of bit strings $X=(x_1, \dots, x_H)\in\set{0,1}^H$, $Y=(y_1,\dots,y_H)\in\set{0,1}^H$ and any $\eps\in(0,1)$, the MDP $M_{X, Y, \eps}$ is defined as follows.
        \begin{enumerate}
            \item The state space is $S_0=\set{s_{0}}$ and $S_h=\set{s_{0,h}, s_{1,h}}$ for all $1\le h\le H$. The MDP starts at $s_0$ deterministically and terminates at $s_{0,H}$ or $s_{1,H}$.
            \item The action space is $A_h=\set{0, 1}$ for all $0\le h\le H$.
            \item The transition is defined as follows:
            \begin{itemize}
                \item $s_{0}$ transitions to $s_{0,1}$ or $s_{1,1}$ with probability at least $1/2$ each, regardless of the action taken.
                \item For any $1\le h\le H-1$, $s_{y_h, h}$ transitions to $s_{y_{h+1}, h+1}$ deterministically if $a_h=x_h\oplus y_h$ (``correct state'' in combination lock), and transitions to $s_{1-y_{h+1}, h+1}$ deterministically if $a_h=1 - x_h\oplus y_h$.
                \item For any $1\le h\le H-1$, $s_{1-y_h, h}$ transitions to $s_{1-y_{h+1}, h+1}$ deterministically regardless of the action taken  (``wrong state'' in combination lock).
            \end{itemize}
            \item The reward is $r_h\equiv 0$ for all $0\le h\le H-1$. At step $H$, we have
            \begin{itemize}
                \item $r_H(s_{y_H, H}) \sim \Ber(1/2+\eps)$,
                \item $r_H(s_{1-y_H, H}) \sim \Ber(1/2 - \eps)$.
            \end{itemize}
        \end{enumerate}
    \end{definition}

    A visualization for the MDP specified by $X$, $Y$ and $\eps$ is shown in Figure~\ref{fig:mdp-long}.
	\begin{figure}[ht]
		\centering
		\includegraphics[width=0.8\textwidth]{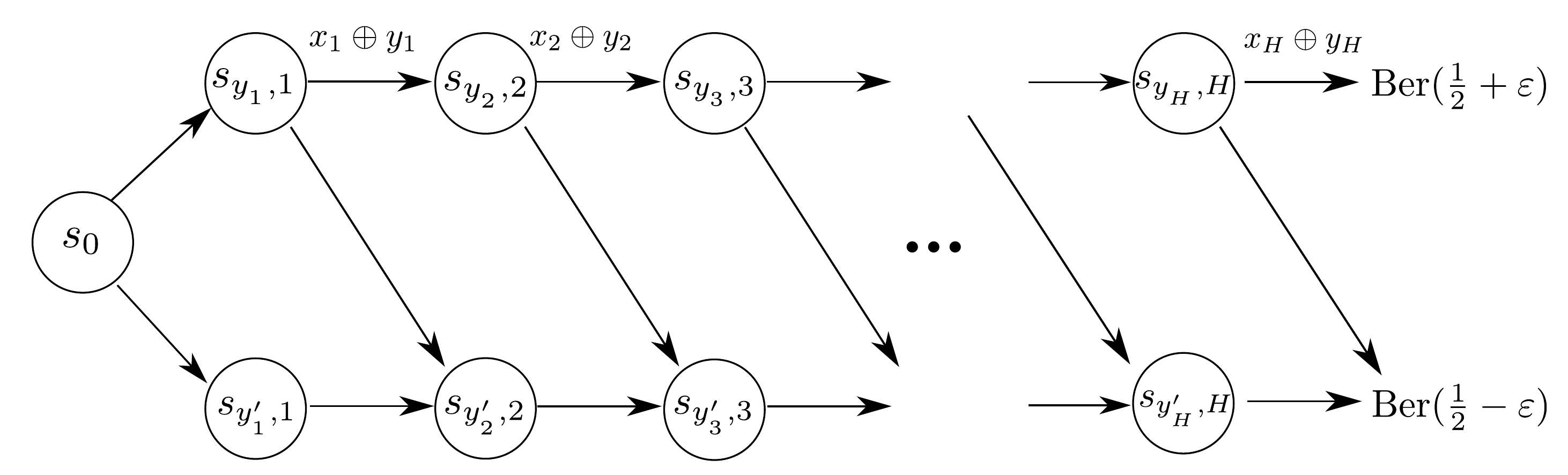}
		\caption{$M(X,Y)$: ``Combination lock'' MDP specified by $X$ and $Y$. For $y\in\{0,1\}$, $y'$ stands for $1-y$.}
		\label{fig:mdp-long}
	\end{figure}

    It is straightforward to see that the optimal value function of this MDP is $1/2(1/2 + \eps) + 1/2(1/2 - \eps)=1/2$, and the only way to achieve higher reward than $1/2-\eps$ is by following the path of ``good states'': $(s_0,s_{y_1,1},\cdots,s_{y_h,h},\cdots,s_{y_H,H})$. The corresponding optimal policy is $\pi^\ast(s_{w,h})=w\oplus x_h$, which is independent of $Y$.  

    \paragraph{Random sequence of MDPs is as hard as a $2^H$-armed bandit.}
    We now consider any fixed (but unknown) $X\in\set{0,1}^H$ and draw $K$ independent samples $Y_k\sim {\sf Unif}(\set{0,1}^H)$ for $1\le k\le K$. We argue that if we provide $M_k\defeq M_{X, Y_k, \eps}$ in episode $k$ (with some appropriate choice of $\eps$), then the problem is as hard as a $2^H$-armed bandit problem with (minimum) suboptimality gap $\eps$, and thus must have the desired regret lower bound.

    Our first claim is that, on average over $Y_k$, the trajectory seen by the algorithm is equivalent (equal in distribution) to the following ``completely random'' MDP: each state $s_{\set{0,1}, h}$ transitions to $s_{\set{0,1}, h+1}$ with probability at least $1/2$ regardless of the actions taken; and the reward is $r_H\sim \Ber(1/2)$ if $A=X\oplus Y$ and $r_H\sim \Ber(1/2-\eps)$ if $A\neq X\oplus Y$, where $A=\set{a_1,\dots,a_h}$ are the actions taken in steps 1 through $H$. Indeed, consider the transition starting from $s_{y_h,h}$. Since $y_{h+1}\sim\Ber(1/2)$, the transition probability to $s_{0,h+1}$ and $s_{1,h+1}$ must be $1/2$ each, regardless of the action taken. The claim about the reward follows from the definition of the MDP.

    We now construct a bandit instance, and show that solving this bandit problem can be reduced to online learning in the sequence of MDPs above. The bandit instance has $2^H$ arms indexed by $\set{0,1}^H$. The arm indexed by $X$ gives reward $\Ber(1/2)$, and otherwise the reward is $\Ber(1/2-\eps)$. Now, for any algorithm solving the adversarial MDP problem, consider the following induced algorithm for the bandit problem.

    \begin{algorithm}[h]
        \caption{Reducing bandits to adversarial MDPs}
        \label{algorithm:bandit-to-adversarial-mdp}
        \begin{algorithmic}[1]
            \For{$k=1, \dots, K$}
                \State Sample $Y\sim {\sf Unif}(\set{0,1}^H)$.
                \State Simulate the adversarial MDP algorithm by showing the trajectory $(s_0,s_{y_1,1},\dots,s_{y_H, H})$.
                \State Denote the action sequence by $A=(a_1,\dots,a_H)$.
                \State Play $A\oplus Y$ in the bandit environment.
                \State Show the received bandit reward to the adversarial MDP algorithm as the last step reward. 
            \EndFor
        \end{algorithmic}
    \end{algorithm}
    We now argue that the interaction seen by the adversarial MDP algorithm is identical in distribution to the sequence $M_{X, Y_k,\eps}$. The trajectory is drawn from a uniform distribution, which is the same as that generated by $M_{X, Y_k,\eps}$. The reward is high, i.e. $\Ber(1/2)$, if and only if $A\oplus Y=X$, which is equivalent to $A=X\oplus Y$. 
    This is also the case in the adversarial MDP problem, since playing the action sequence $X\oplus Y$ corresponds to playing the optimal policy $\pi^*(s_{y_h,h})=x_h\oplus y_h$.

    Therefore, the regret achieved by the induced algorithm in the bandit environment would be equal (in distribution) to the regret achieved by this algorithm in the adversarial MDP environment. Applying classical lower bounds on stochastic bandits~\citep[Chapter 15]{lattimore2020bandit} (which corresponds to taking $\eps=\eps_{H,K}\defeq \min\set{\sqrt{2^H/K}, 1/4}$), we obtain
    \begin{equation*}
    \sup_{X\in \set{0,1}^H} \E_{Y_1,\dots,Y_k\sim {\sf Unif}(\set{0,1}^H)} \brac{ \sum_{k=1}^K \left(V^\ast_{M_{X, Y_k,\eps_{H,K}}}(s_0)  - \E_{\mu^k} V^{\mu^k}_{M_{X, Y_k, \eps_{H,K}}}(s_0)\right) } \ge \Omega(\min\set{\sqrt{2^HK}, K}),
    \end{equation*}
    where $\E_{\mu^k}$ denotes the randomness in the algorithm execution (which includes the randomness of the realized transitions and rewards that were used by the algorithm to determine $\mu^k$). Note that for the MDP $M_{X,Y_k,\epsilon_H,T}$, the optimal policy is dictated by $X$ and independent of $Y_k$ (hence independent of $k$). Thus, the previous lower bound can rewritten as a comparison with the best policy in hindsight:
    \begin{equation*}
        \sup_{X\in \set{0,1}^H}\sup_\mu \E_{Y_1,\dots,Y_k\sim {\sf Unif}(\set{0,1}^H)} \brac{ \sum_{k=1}^K \left(V^\mu_{M_{X, Y_k,\eps_{H,K}}}(s_0)  - \E_{\mu^k} V^{\mu^k}_{M_{X, Y_k, \eps_{H,K}}}(s_0)\right) } \ge \Omega(\min\set{\sqrt{2^HK}, K}).
    \end{equation*}
 
    \paragraph{The adversarial MDP problem is as hard as the above random sequence of MDPs.}
    Define $\cM \defeq \set{M_{X, Y, \eps_{H,K}}:X, Y \in \set{0,1}^H}$. As the minimax regret is lower bounded by the average regret over any prior distribution of MDPs, the above lower bound implies the following minimax lower bound
    \begin{equation*}
        \sup_{M_k\in\cM} \sup_\mu\brac{ \sum_{k=1}^K \left(V^\mu_{M_k}(s_0)  - \E_{\mu^k} V^{\mu^k}_{M_k}(s_0)\right) } \ge \Omega(\min\set{\sqrt{2^HK}, K})
    \end{equation*}
    for any adversarial MDP algorithm.
\end{proof}

\begin{proof}[Proof of Theorem 1]
    With Lemma~\ref{lem:lb-adv-mdp} in hand, we are in a position to prove the main theorem.

    Our proof follows by defining a two-player Markov game and a set of min-player policies $\set{\nu^k}$ such that the transitions and rewards seen by the max-player are exactly equivalent to the MDP $M_{X, Y_k, \eps_{H,K}}$ constructed in Lemma~\ref{lem:lb-adv-mdp}. Indeed, we augment the MDP $M_{X, Y_k, \eps_{H, K}}$ with a set of min-player actions $\cB_h=\set{1,2,3,4}$, and redefine the transition such that from any $s_{i,h}$ where $i\in\set{0,1}$ and $1\le h\le H-1$, the Markov game transitions according to Table~\ref{tab:trans}.
    \begin{table}[h]
        \center
        \begin{tabular}{l|l|l|l|l}
            $a$/$b$ & $1$           & $2$             & $3$             & $4$             \\ \hline
            $0$   & $s_{i,h+1}$ & $s_{1-i,h+1}$ & $s_{i,h+1}$   & $s_{1-i,h+1}$ \\ \hline
            $1$   & $s_{i,h+1}$ & $s_{1-i,h+1}$ & $s_{1-i,h+1}$ & $s_{i,h+1}$  
        \end{tabular}
        \caption{transition function of the state $s_{i,h}$ for the hard instance of Markov games.}
        \label{tab:trans}
    \end{table}

    Such an action set $\cB_h$ is powerful enough to reproduce all the possible transitions in the original single-player MDP. We then define $\nu^k$ as the policy such that the transition follows exactly $M_{X, Y_k}$. The reward function is determined only by states and thus remains the same. 
    Therefore, Lemma~\ref{lem:lb-adv-mdp} implies the following one-sided regret bound for the max-player: 
    \begin{equation*}
        \sup_{\nu^k} \sup_{\mu} \sum_{k=1}^K \left(V^{\mu, \nu^k}(s_0)  - \E_{\mu^k} V^{\mu^k, \nu^k}(s_0) \right) \ge \Omega(\min\set{\sqrt{2^HK}, K}),
    \end{equation*}
    which is the desired result.
    \iffalse
    Therefore, Lemma~\ref{lem:lb-adv-mdp} now implies the following one-sided regret bound for the max-player:
    \begin{equation*}
    \sup_{\nu^k}  \brac{ \sum_{k=1}^K \left(\max_{\mu} V^{\mu, \nu^k}(s_0)  - \E_{\mu^k} V^{\mu^k, \nu^k}(s_0)\right) } \ge \Omega(\min\set{\sqrt{2^HK}, K}).
    \end{equation*}
    Finally, note that $\sum_{k=1}^K \max_{\mu} V^{\mu, \nu^k}(s_0) = \max_{\mu} \sum_{k=1}^K V^{\mu, \nu^k}(s_0)$; this is because for any $\nu^k$ the max over $\mu$ can be achieved at a single optimal policy $\mu^\ast$ regardless of $\nu^k$:
    \begin{equation*}
    \mu^\ast_h(s_{y, h}) = x_h \oplus y,~~~\textrm{for~all}~y\in \set{0,1},~h\in [H].
    \end{equation*}
    Therefore, the above bound also implies
    \begin{equation*}
    \sup_{\nu^k}  \brac{ \max_{\mu} \sum_{k=1}^K \left(V^{\mu, \nu^k}(s_0)  - \E_{\mu^k} V^{\mu^k, \nu^k}(s_0) \right)} \ge \Omega(\min\set{\sqrt{2^HK}, K}),
    \end{equation*}
    which is the desired result.
    \fi
\end{proof}

\section{Proof for the \nashvol{} algorithm}

Throughout this section, let $\iota = \log (\nicefrac{HSAK}{p})$. The following lemma summarizes the key properties of the choice of the learning rate $\alpha_t$, which are used in the proof below.

\begin{lemma}[{\citep[Lemma 4.1]{jin2018q}}]  \label{lem:ppt-alpha}
    The following properties hold for $\alpha_{t}^{i}$.
    \begin{enumerate}
        \item $\nicefrac{1}{\sqrt{t}} \le \sum_{i=1}^{t} \nicefrac{\alpha_t^i}{\sqrt{i}} \le \nicefrac{2}{\sqrt{t}}$ for all $t \ge 1$. \label{ppt:1}
        \item $\sum_{i=1}^{t} (\alpha_t^i)^2 \le \max_{i\in [t]} \alpha_{t}^{i} \le \nicefrac{2 GH}{t}$ for all $t\ge 1$. \label{ppt:2}
        \item $\sum_{t=i}^{\infty} \alpha_t^i = 1 + \nicefrac{1}{GH}$ for all $i \ge 1$. \label{ppt:3}
    \end{enumerate}    
\end{lemma}

\subsection{Upper confidence bound on the minimax value function}

\begin{lemma}[V-learning lemma]  \label{lem:v-est-err} 
    In Algorithm~\ref{alg:nash-v}, let $t = N_{h}^{k}(s)$ and suppose state $s\in \cS_h$ was previously visited at episodes $k^1,\ldots, k^t < k$ at the $h$th step. For any $p\in (0, 1)$, let $\iota = \log (\nicefrac{HSAK}{p})$. Choose $\eta_t = \sqrt{\nicefrac{GH\log A}{At}}$. Then with probability at least $1-p$, for any $t\in [K]$, $h\in [H]$ and $s\in \cS_h$, there exists a constant $c$ such that 
    \begin{align} \label{eqn:v-est-err}
        \max_{\mu\in \Delta_{\cA}} \sum_{i=1}^{t} \alpha_{t}^{i} \D_{\mu, \nu_{h}^{k^i}} \left[ r_h + \Trans_h \uV_{h+1}^{k^i} \right] (s) 
        - \sum_{i=1}^{t} \alpha_{t}^{i} \left( r_h (s, a_{h}^{k^i}, b_{h}^{k^i} ) + \uV_{h+1}^{k^i} (s_{h+1}^{k^i}) \right)
        \le c \sqrt{GH^3 A\iota/t}.
    \end{align}
\end{lemma}

\begin{proof}
    By the Azuma-Hoeffding inequality and Lemma~\ref{lem:ppt-alpha},
    \begin{align*}
        \sum_{i=1}^t{\alpha _{t}^{i}\D_{\mu _{h}^{k^i} \times \nu _{h}^{k^i}} \left( r_h +\Trans_hV_{h+1}^{k^i} \right)\left( s \right) } - \sum_{i=1}^t{\alpha _{t}^{i}\left[ r_h\left( s,a_{h}^{k^i},b_{h}^{k^i} \right) +V_{h+1}^{k^i}\left( s_{h+1}^{k^i} \right) \right]}\le 2 \sqrt {GH^3\iota /t}.    
    \end{align*}
    So we only need to bound 
    \begin{align}\label{eq:regret_definition}
        R_t^{\ast} := \underset{\mu\in \Delta_{\cA}}{\max}\sum_{i=1}^t{\alpha _{t}^{i}\D_{\mu \times \nu _{h}^{k^i}} \left( r_h +\Trans_hV_{h+1}^{k^i} \right)\left( s \right) } - \sum_{i=1}^t{\alpha _{t}^{i}\D_{\mu _{h}^{k^i} \times \nu _{h}^{k^i}} \left( r_h +\Trans_hV_{h+1}^{k^i} \right)\left( s \right) }. 
    \end{align}
    By taking $w_i = \alpha_t^i$ in~\citep[Lemma 17]{bai2020near},
    \begin{align*}
        R_{t}^{\ast}\le & \frac{3H\alpha_{t}^{t}\log A}{\eta _t}+\frac{3A}{2}\nlsum_{i=1}^t{\eta_i}\alpha _{t}^{i}+\sqrt{2\iota \nlsum_{i=1}^t{\left( \alpha _{t}^{i} \right) ^2}} \\
        \overset{(i)}{\le} & 3H\alpha_{t}^{t}\sqrt{\tfrac{At\log A}{GH}} + \frac{3}{2}\sqrt{GH A\log A}\nlsum_{i=1}^t{\frac{\alpha_{t}^{i}}{\sqrt{i}}} + \sqrt{2\iota \nlsum_{i=1}^t{\left( \alpha_{t}^{i} \right) ^2}} \\
        \overset{(ii)}{\le} & 3H\tfrac{GH+1}{GH+t}\sqrt{\tfrac{At\log A}{GH}} + 3\sqrt{\tfrac{GH A \log A}{t}}+2\sqrt{\tfrac{GH\iota}{t}} \\
        \le & c\sqrt{\tfrac{GHA\iota}{t}}
    \end{align*}
    for some constant $c$, where $(i)$ is by setting $\eta_t = \sqrt{\frac{GH\log A}{At}}$ and $(ii)$ is by Lemma~\ref{lem:ppt-alpha}. Taking union bound w.r.t. all $(t,s,h) \in [K] \times \cS \times [H]$ concludes the proof.

    We comment that the quantity $R_{t}^{*}$ is actually $H$ times the LHS in the inequality of~\citep[Lemma 17]{bai2020near}. See Appendix F and Algorithm 9 in \cite{bai2020near} for a detailed reduction from MG to adversarial bandit problem. Furthermore, in~\citep{bai2020near} there are actually two parameters $\eta_t $ and $\gamma_t$. Here we just take $\gamma_t= \eta_t$ for simplicity. Finally, the proof of~\citep[Lemma 17]{bai2020near} requires that $\eta_i \le 2\gamma_i$ for all $i\le t$~\citep[Lemma 19]{bai2020near} and that $\gamma_t$ is nondecreasing in $t$~\citep[Lemma 21]{bai2020near}, which are both satisfied by our specification of $\eta_t$. 
\end{proof}

\begin{lemma}[Upper confidence bound]
    In Algorithm~\ref{alg:nash-v}, for any $p \in (0, 1)$, let $\iota = \log (\nicefrac{HSAK}{p})$ and choose $\beta_t = c\sqrt{GH^3A\iota / t}$ for some large constant $c$. Then with probability at least $1-p$, $\oV{h}(s) \le \uV_{h}^{k}(s)$ for all $k\in [K]$, $h\in [H]$ and $s\in \cS_h$.
\end{lemma}

\begin{proof}
    The proof is similar to that of~\citep[Lemma 15]{bai2020near}, except that we need to deal with an extra parameter $G$ here.

    Let $k_h^{i}(s)$ denote the index of the episode where $s\in \cS_h$ is observed at step $h$ for the $i$th time.
    Where there is no ambiguity, we use $k^i$ as a shorthand for $k_h^{i}(s)$.
    Let $s_{h}^{k}$ be the state actually observed in the algorithm at step $h$ in episode $k$.
    For our choice of $\ubeta_{i}$, we have $\sum_{i=1}^{t} \alpha_{t}^{i} \ubeta_i = \Theta(G H^2 \sqrt{\nicefrac{A\iota}{t}})$ by Lemma~\ref{lem:ppt-alpha}.

    Recall that 
    \begin{align*}
        \uV_{h}^{k}(s) &:= \alpha_{t}^{0} H + \sum_{i=1}^{t} \alpha_{t}^{i} \left( r_{h}(s, a_{h}^{k^i}, b_{h}^{k^i}) + \uV_{h+1}^{k^i}(s_{h+1}^{k^i}) + \ubeta_{i} \right), \\
        \oV{h}(s) &:= \D_{\mu_{h}^{\ast}, \nu_{h}^{\ast}} [r_{h} + \Trans_{h} \oV{h+1}](s).
    \end{align*}

    For $h = H + 1$ the UCB vacuously holds. To apply backward induction, assume that $\oV{h+1} \le \uV_{h+1}^{k}$ holds entrywise.
    Then by definition, for any $s\in \cS_h$, 
    \begin{align*}
        \oV{h}(s) 
        &= \max_{\mu \in \Delta_{\cA_h}} \min_{\nu\in \Delta_{\cB_h}} \D_{\mu, \nu} [r_{h} + \Trans_{h} \oV{h+1}](s) \\
        &\overset{(i)}{=} \max_{\mu \in \Delta_{\cA_h}} \nlsum_{i=1}^{t} \alpha_{t}^{i} \min_{\nu\in \Delta_{\cB_h}} \D_{\mu, \nu} [r_{h} + \Trans_{h} \oV{h+1}](s) \\
        &\le \max_{\mu \in \Delta_{\cA_h}} \nlsum_{i=1}^{t} \alpha_{t}^{i} \D_{\mu, \nu_{h}^{k^i}} [r_{h} + \Trans_{h} \oV{h+1}](s) \\
        &\overset{(ii)}{\le} \max_{\mu \in \Delta_{\cA_h}} \nlsum_{i=1}^{t} \alpha_{t}^{i} \D_{\mu, \nu_{h}^{k^i}} [r_{h} + \Trans_{h} \uV_{h+1}^{k^i}](s) 
        \overset{(iii)}{\le} \uV_{h}^{k}(s), 
    \end{align*}
    where $(i)$ follows from $\sum_{i=1}^{t} \alpha_{t}^{i} = 1$, in $(ii)$ we apply the induction assumption, and $(iii)$ holds with probability at least $1-p$ by the V-learning lemma (Lemma \ref{lem:v-est-err}) and that $\sum_{i=1}^t \alpha_t^i \beta_t = \Theta(\sqrt{GH^3 A\iota/t})$
    because of our choice of $\beta_t$ and Property~\ref{ppt:1} of $\{\alpha_t^i\}$ in Lemma~\ref{lem:ppt-alpha}.
    Inductively we have $\oV{h}(s) \le \uV_{h}^{k}(s)$ for all $k\in [K]$, $h\in [H]$ and $s\in \cS_h$.
\end{proof}

\subsection{Proof of Theorem~\ref{thm:nash-v}}

\begin{proof}
    In the proof below, we use `$\lesssim$' to denote `$\le$' hiding some constants.
    Recall that 
    \begin{align*}
        V_{h}^{\mu^{k}, \nu^{k}}(s_{h}^{k}) 
        = \D_{\mu_{h}^{k}, \nu_{h}^{k}} [r_{h} + \Trans_{h} V_{h+1}^{\mu^{k}, \nu^{k}}](s_{h}^{k}).
    \end{align*}
    Then define $\delta_{h}^{k} := (\uV_{h}^{k} - V_{h}^{\mu^{k}, \nu^{k}})(s_{h}^{k})$. By definition, 
    \begin{align*}
        \delta_{h}^{k} 
        &= \alpha_{t}^{0} H + \sum_{i=1}^{t} \alpha_{t}^{i} \left( r_{h}(s_{h}^{k}, a_{h}^{k^i}, b_{h}^{k^i}) + \uV_{h+1}^{k^i}(s_{h+1}^{k^i}) + \ubeta_{i} \right) - \D_{\mu_{h}^{k}, \nu_{h}^{k}} [r_{h} + \Trans_{h} V_{h+1}^{\mu^{k}, \nu^{k}}](s_{h}^{k}) \\
        &\overset{(i)}{=} \alpha_{t}^{0} H + \sum_{i=1}^{t} \alpha_{t}^{i} \left( r_{h}(s_{h}^{k}, a_{h}^{k^i}, b_{h}^{k^i}) + \uV_{h+1}^{k^i}(s_{h+1}^{k^i}) + \ubeta_{i} \right) - \sum_{i=1}^{t} \alpha_{t}^{i} \D_{\mu^{k^i}, \nu^{k^i}} [r_h + \Trans_{h} \uV_{h+1}^{k^i}](s_{h}^{k}) \\
        &\qquad \qquad  + \sum_{i=1}^{t} \alpha_{t}^{i} \D_{\mu^{k^i}, \nu^{k^i}} [r_h + \Trans_{h} \uV_{h+1}^{k^i}](s_{h}^{k}) - \D_{\mu_{h}^{k}, \nu_{h}^{k}} [r_{h} + \Trans_{h} V_{h+1}^{\mu^{k}, \nu^{k}}](s_{h}^{k}) \\
        &\overset{(ii)}{\lesssim} \alpha_{t}^{0} H + \sqrt{\tfrac{GH^3 A\iota}{t}} 
        + \sum_{i=1}^{t} \alpha_{t}^{i} \D_{\mu^{k^i}, \nu^{k^i}} [r_h + \Trans_{h} \uV_{h+1}^{k^i}](s_{h}^{k}) - \D_{\mu_{h}^{k}, \nu_{h}^{k}} [r_{h} + \Trans_{h} V_{h+1}^{\mu^{k}, \nu^{k}}](s_{h}^{k}),
    \end{align*}
    where in $(i)$ we add and subtract the same term, and $(ii)$ follows from the property of $\beta_i$ that $\sum_{i=1}^{t} \alpha_{t}^{i} \ubeta_i = \Theta(\sqrt{\nicefrac{GH^3 A\iota}{t}})$ and the fact that by the Azuma-Hoeffding inequality and Property~\ref{ppt:2} of Lemma~\ref{lem:ppt-alpha}, 
    \begin{align*}
        \sum_{i=1}^{t} \alpha_{t}^{i} \left( r_{h}(s_{h}^{k}, a_{h}^{k^i}, b_{h}^{k^i}) + \uV_{h+1}^{k^i}(s_{h+1}^{k^i}) \right) - \sum_{i=1}^{t} \alpha_{t}^{i} \D_{\mu^{k^i}, \nu^{k^i}} [r_h + \Trans_{h} \uV_{h+1}^{k^i}](s_{h}^{k}) \lesssim \sqrt{\tfrac{G H^3\iota}{t}}.
    \end{align*}
    By the same regrouping technique as that in~\citep{jin2018q}, 
    \begin{align*}
        \sum_{k=1}^{K} \sum_{i=1}^{t} \alpha_{t}^{i} \D_{\mu^{k^i}, \nu^{k^i}} [r_h + \Trans_{h} \uV_{h+1}^{k^i}](s_{h}^{k}) 
        &\le \sum_{k'=1}^{K} \D_{\mu^{k'}, \nu^{k'}} [r_h + \Trans_{h} \uV_{h+1}^{k'}](s_{h}^{k}) \sum_{t=n_{h}^{k'}}^{\infty} \alpha_{t}^{n_{h}^{k'}} \\
        &\le (1 + \tfrac{1}{GH}) \sum_{k=1}^{K} \D_{\mu^{k}, \nu^{k}} [r_h + \Trans_{h} \uV_{h+1}^{k}](s_{h}^{k}).
    \end{align*}
    Substituting the above back into the bound on $\delta_{h}^{k}$ and taking sum over $k\in [K]$, we obtain 
    \begin{align*}
        \sum_{k=1}^{K} \delta_{h}^{k} 
        &\lesssim \sum_{k=1}^{K} \left( \alpha_{t}^{0} H + \sqrt{\tfrac{GH^3 A\iota}{t}} + (1 + \tfrac{1}{GH}) \D_{\mu^{k}, \nu^{k}} [r_h + \Trans_{h} \uV_{h+1}^{k}](s_{h}^{k}) - \D_{\mu_{h}^{k}, \nu_{h}^{k}} [r_{h} + \Trans_{h} V_{h+1}^{\mu^{k}, \nu^{k}}](s_{h}^{k}) \right) \\
        &\overset{(i)}{=} \sum_{k=1}^{K} \left( \alpha_{t}^{0} H + \sqrt{\tfrac{GH^3 A\iota}{t}} + (1 + \tfrac{1}{GH}) (\delta_{h+1}^{k} + \gamma_{h}^{k}) + \tfrac{1}{GH} \D_{\mu_{h}^{k}, \nu_{h}^{k}}[r_{h} + \Trans_{h} V_{h+1}^{\mu^{k}, \nu^{k}}](s_{h}^{k}) \right) \\
        &\overset{(ii)}{\le} \sum_{k=1}^{K} \left( \alpha_{t}^{0} H + \sqrt{\tfrac{GH^3 A\iota}{t}} + (1 + \tfrac{1}{GH}) (\delta_{h+1}^{k} + \gamma_{h}^{k}) + \tfrac{1}{G} \right),
    \end{align*}
    where in $(i)$ we define the martingale difference term $\gamma_{h}^{k} := \D_{\mu_{h}^{k}, \nu_{h}^{k}} [\Trans_{h} (\uV_{h+1}^{k} - V_{h+1}^{\mu^{k}, \nu^{k}})](s_{h}^{k}) - (\uV_{h+1}^{k} - V_{h+1}^{\mu^{k}, \nu^{k}})(s_{h+1}^{k})$ and $(ii)$ follows from that 
    \begin{align*}
        \D_{\mu_{h}^{k}, \nu_{h}^{k}}[r_{h} + \Trans_{h} V_{h+1}^{\mu^{k}, \nu^{k}}](s_{h}^{k}) \le H.
    \end{align*}
    Recursively, 
    \begin{align*}
        \sum_{k=1}^{K} \delta_{1}^{k} 
        &\lesssim (1 + \tfrac{1}{GH})^{H} \sum_{k=1}^{K} \sum_{h=1}^{H} \left( \alpha_{t}^{0} H + \sqrt{\tfrac{GH^3 A\iota}{t}} + (1 + \tfrac{1}{GH}) \gamma_{h}^{k} + \tfrac{1}{G} \right).
    \end{align*}
    Now we bound each term in $\sum_{k=1}^{K} \delta_{1}^{k}$ separately by standard techniques in~\citep{jin2018q, xie2020learning}:
    \begin{align*}
        & \sum_{k=1}^{K} \alpha_{n_{h}^{k}}^{0} H 
        \le \sum_{k=1}^{K} H \cdot \I(n_{h}^{k} = 0) 
        \le HS, \\
        & \sum_{k=1}^{K} \sqrt{\tfrac{GH^3 A\iota}{n_{h}^{k}}}
        = GH^2 \sqrt{A\iota} \sum_{k=1}^{K} \sqrt{\tfrac{1}{n_{h}^{k}}} 
        \le \sqrt{GH^3 A\iota} \sum_{s\in \cS_h} \sum_{n=1}^{n_h^K(s)} \sqrt{\tfrac{1}{n}}
        \lesssim \sqrt{GH^3 S A K\iota}), \\
        & \sum_{k=1}^{K} \sum_{h=1}^{H} \gamma_{h}^{k}
        \lesssim \sqrt{H^3 K\iota},
    \end{align*}
    where the second line follows from a pigeonhole argument and the third line follows from the Azuma-Hoeffding inequality.
    Combining the above bounds, we obtain 
    \begin{align*}
        \Regret(K) \le \sum_{k=1}^{K} \delta_{1}^{k} 
        \lesssim H^2 S + \sqrt{GH^5 SAK\iota} + G^{-1} KH. 
    \end{align*}

    If $K \ge H^3 SA$ then we take we take $G = \frac{1}{H}(\frac{K}{SA})^{\nicefrac{1}{3}}$; otherwise we take $G = K^{\tfrac{1}{3}}$.
    Then the following regret bounds holds:
    \begin{align*}
        \Regret(K) = 
        \begin{cases}
            \bigotilde\bigl( H^2 S^{\frac{1}{3}} A^{\frac{1}{3}} K^{\frac{2}{3}} + H^2 S \bigr), \text{ if } K \ge H^3SA, \\
            \bigotilde\bigl( \sqrt{H^5 SA} K^{\frac{2}{3}} + H^2 S \bigr), \text{ otherwise.} \\
        \end{cases}
    \end{align*}
\end{proof}

\section{The \nashqol{} Algorithm}
\label{sec:nashq}

When explaining the intuition behind the \nashvol{}
in Section~\ref{sec:vlearning}, we mentioned that learning a Q-table will result in a regret bound depending on $AB$. This is clear for the other algorithms we mentioned in the literature. However, the regret bounds of Q-learning-type algorithms have not been studied to our best knowledge. In this section, we study a Q-learning-type algorithm for online MGs. We formalize \nashvol{} in Algorithm~\ref{alg:nash-q}, which is similar to the Optimistic Nash Q-learning (\nashqsp{}) algorithm in~\citep{bai2020near}. We emphasize that since learning a Q-table requires knowing the opponents' actions, \nashqol{} only works for informed MGs, but not for unknown MGs. 

\begin{algorithm}[htbp]
    \caption{Optimistic Nash Q-learning for Online Learning (\nashqol{})}
    \label{alg:nash-q}
    \begin{algorithmic}[1]
        \State {\bfseries Require:} Learning rate $\{\alpha_t\}_{t\ge 1}$, exploration bonus $\{\beta_t\}_{t\ge 1}$
        \State {\bfseries Initialize:} for any $(s, a, b, h)$,
        $\uQ_{h}(s, a, b)\setto H$, $N_{h}(s, a, b)\setto 0$, $\mu_h(a\vert s) \leftarrow 1/A$
        \For{episode $k=1,\dots,K$}
            \State Receive $s_1$
            \For{step $h=1,\dots, H$}
                \State Take action $a_h \sim \mu_h(\cdot| s_h)$
                \State Observe action $b_h$, reward $r_h(s_h, a_h, b_h)$ and next state
                $s_{h+1}$
                \State $t = N_{h}(s_h, a_h, b_h)\setto N_{h}(s_h, a_h, b_h) + 1$
                \State $\uQ_h(s_h, a_h, b_h) \setto (1-\alpha_t)\uQ_h(s_h, a_h, b_h)+ \alpha_t(r_h(s_h, a_h, b_h)+\uV_{h+1}(s_{h+1})+\beta_t)$ \label{line:uQ_update}
                \State Solve the NE $(\mu_{h}(\cdot, \vert s_h), \nu_{h}(\cdot, \vert s_{h}))$ of the matrix game with payoff matrix $Q_{h}^{k}(s_{h}, \cdot, \cdot)$
                \State $\uV_h(s_h) \leftarrow (\D_{\mu_{h}\times \nu_{h} }\uQ_h)(s_h)$
            \EndFor
        \EndFor
    \end{algorithmic}
\end{algorithm}

In Algorithm~\ref{alg:nash-q}, we set $\alpha_t := \nicefrac{H + 1}{H + t}$. 
As in the analysis of \nashvol{}, below we use a superscript $k$ to signify the corresponding quantities at the beginning of the $k$th episode. 
The following lemma claims that $\uQ_h^k$ and $\uV_h^k$ are the entrywise upper confidence bounds of $\oQ{h}$ and $\oV{h}$ for all $k \in [K]$ and $h\in [H]$; see the proof of~\citep[Lemma 3]{bai2020near} for its proof.
\begin{lemma}[Upper confidence bounds]  \label{lem:q-ucb}
    In Algorithm~\ref{alg:nash-q}, for any $p\in (0, 1)$, $\iota = \log (\nicefrac{HSAK}{p})$ and choose $\beta_t = c \sqrt{H^3 \iota / t}$ for some large constant $c$. Then with probability at least $1-p$, $\oQ{h}(s, a, b) \le \uQ_{h}^{k}(s, a, b)$ and $\oV{h}(s) \le \uV_{h}^{k}(s)$ for all $k\in [K]$, $h\in [H]$ and $(s, a, b) \in \cS \times \cA_h \times \cB_h$.
\end{lemma}

Then for \nashqol{}, we have the following regret guarantees.

\begin{theorem}[Regret bound of \nashqol{}]  \label{thm:nash-q}
    For any $p \in (0, 1)$, let $\iota = \log (\nicefrac{HSAK}{p})$ and choose $\beta_t = c \sqrt{H^3 \iota / t}$ for some large constant $c$. If we run \nashqol{} in a two-player zero-sum MG, then with probability at least $1 - p$, the regret in $K$ episodes satisfies
    \begin{align} \label{eqn:nash-q-regret}
        \Regret(K) 
        = \bigo\left( SABH^2 + \sqrt{H^5 SABK \iota} \right).
    \end{align}
\end{theorem}
 
\begin{proof}
    Let $k_h^{i}(s, a, b)$ denote the index of the episode where $(s, a, b)$ is observed at step $h$ for the $i$th time.
    Where there is no ambiguity, we use $k^i$ as a shorthand for $k_h^{i}(s, a, b)$.
    Let $s_{h}^{k}$ be the state actually observed in the algorithm at step $h$ in episode $k$.

    By defining 
    \begin{align*}
        \ugamma_{h}^{k} &:= \E_{a\sim \mu_{h}^{k}(s_{h}^{k})} [\uQ_{h}^{k}(s_{h}^{k}, a, b_{h}^{k})] - \uQ_{h}^{k}(s_{h}^{k}, a_{h}^{k}, b_{h}^{k}), \\
        \hgamma_{h}^{k} &:= \E_{a\sim \mu_{h}^{k}(s_{h}^{k}), b \sim \omega_{h}^{k}} [Q_{h}^{\mu^{k}, \omega^{k}}(s_{h}^{k}, a, b)] - Q_{h}^{\mu^{k}, \omega^{k}}(s_{h}^{k}, a_{h}^{k}, b_{h}^{k}),
    \end{align*}
    we have 
    \begin{align*}
        & \uV_{h}^{k}(s_{h}^{k}) 
        = \min_{\nu \in \Delta_{\cB}} \E_{a\sim \mu_{h}^{k}(s_{h}^{k}), b \sim \nu} [\uQ_{h}^{k}(s_{h}^{k}, a, b)] 
        \le \E_{a\sim \mu_{h}^{k}(s_{h}^{k})} [\uQ_{h}^{k}(s_{h}^{k}, a, b_{h}^{k})] 
        = \uQ_{h}^{k}(s_{h}^{k}, a_{h}^{k}, b_{h}^{k}) + \ugamma_{h}^{k}, \\
        & V_{h}^{\mu^{k}, \omega^{k}}(s_{h}^{k}) 
        = \E_{a\sim \mu_{h}^{k}(s_{h}^{k}), b \sim \omega_{h}^{k}} [Q_{h}^{\mu^{k}, \omega^{k}}(s_{h}^{k}, a, b)]
        = Q_{h}^{\mu^{k}, \omega^{k}}(s_{h}^{k}, a_{h}^{k}, b_{h}^{k}) + \hgamma_{h}^{k}.
    \end{align*}

    Define $\delta_{h}^{k} := \uV_{h}^{k}(s_{h}^{k}) - V_{h}^{\mu^{k}, \omega^{k}}(s_{h}^{k})$ and $\phi_{h}^{k} := \uV_{h}^{k}(s_{h}^{k}) - \oV{h}(s_{h}^{k})$.
    Then 
    \begin{align*}
        \delta_{h}^{k} 
        \le \uQ_{h}^{k}(s_{h}^{k}, a_{h}^{k}, b_{h}^{k}) + \ugamma_{h}^{k} - Q_{h}^{\mu^{k}, \omega^{k}}(s_{h}^{k}, a_{h}^{k}, b_{h}^{k}) - \hgamma_{h}^{k}.
    \end{align*}

    In Algorithm~\ref{alg:nash-q}, for any $k\in [K]$, $h\in [H]$ and $(s, a, b) \in \cS_h \times \cA_h \times \cB_h$, let $t := N_{h}^{k}(s, a, b)$ and suppose $(s, a, b)$ is previously visited at episodes $k^1, \cdots, k^t \le k$. Then we can rewrite $\uQ_h^k(s, a, b)$ as 
    \begin{align*}
        \uQ_{h}^{k}(s, a, b) = \alpha_{t}^{0} H + \sum_{i=1}^{t} \alpha_{t}^{i} \left( r_{h}(s, a, b) + \uV_{h+1}^{k^i}(s_{h+1}^{k^i}) + \beta_{i} \right), 
    \end{align*}
    and recall that 
    \begin{align*}
        \oQ{h}(s, a, b) = r_{h}(s, a, b) + \Trans_{h} \oV{h+1}(s, a, b).
    \end{align*}
    Then the difference between $\uQ_{h}^{k}$ and $Q_{h}^{\mu^{k}, \omega^{k}}$ at $(s_{h}^{k}, a_{h}^{k}, b_{h}^{k})$ satisfies 
    \begin{align*}
        (\uQ_{h}^{k} - Q_{h}^{\mu^{k}, \omega^{k}}) (s_{h}^{k}, a_{h}^{k}, b_{h}^{k}) 
        & \overset{(i)}{=} (\uQ_{h}^{k} - \oQ{h} + \oQ{h} - Q_{h}^{\mu^{k}, \omega^{k}}) (s_{h}^{k}, a_{h}^{k}, b_{h}^{k}) \\
        & \overset{(ii)}{\le} \alpha_{t}^{0} H + \sum_{i=1}^{t} \alpha_{t}^{i} \phi_{h+1}^{k^i} + 2\tbeta_{t} + \Trans_{h} (\oV{h+1} - V_{h+1}^{\mu^{k}, \omega^{k}})(s_{h}^{k}, a_{h}^{k}, b_{h}^{k}) \\
        & \overset{(iii)}{=} \alpha_{t}^{0} H + \sum_{i=1}^{t} \alpha_{t}^{i} \phi_{h+1}^{k^i} + 2\tbeta_{t} + \delta_{h+1}^{k} - \phi_{h+1}^{k} + \zeta_{h}^{k}, 
    \end{align*}
    where in $(i)$ we add and subtract the same term, in $(ii)$ we define $\tbeta_t := \sum_{i=1}^{t} \alpha_{t}^{i} \beta_{i} = \bigo(\sqrt{H^3 \iota / t})$ and by the Azuma-Hoeffding inequality we have 
    \begin{align*}
        \left| \sum_{i=1}^{t} \alpha_{t}^{i} \bigl(\Trans_{h} \oV{h+1}(s, a, b) - \uV_{h+1}^{k^i}(s_{h+1}^{k^i}) \bigr) \right| \le 2H\sqrt{ \iota \sum_{i=1}^{t} (\alpha_{t}^{i})^2 } = \bigo\left(\sqrt{\frac{H^3 \iota}{t}}\right) \overset{\text{choice of $\beta_i$}}{=\joinrel=\joinrel=\joinrel=} \tbeta_{t},
    \end{align*}
    and in $(iii)$ we define 
    \begin{align*}
        \zeta_{h}^{k} := \Trans_{h} (\oV{h+1} - V_{h+1}^{\mu^{k}, \omega^{k}})(s_{h}^{k}, a_{h}^{k}, b_{h}^{k}) - (\oV{h+1} - V_{h+1}^{\mu^{k}, \omega^{k}})(s_{h+1}^{k}).
    \end{align*}
    Therefore, 
    \begin{align*}
        \delta_{h}^{k} \le \delta_{h+1}^{k} + \alpha_{t}^{0} H + \sum_{i=1}^{t} \alpha_{t}^{i} \phi_{h+1}^{k} + 2\tbeta_{t} - \phi_{h+1}^{k} + \zeta_{h}^{k} + \ugamma_{h}^{k} - \hgamma_{h}^{k}.
    \end{align*}
    Recursively, 
    \begin{align}  \label{eqn:delta1k}
        \delta_{1}^{k} \le \sum_{h=1}^{H} \left( \alpha_{t}^{0} H + \sum_{i=1}^{t} \alpha_{t}^{i} \phi_{h+1}^{k} + 2\tbeta_{t} - \phi_{h+1}^{k} + \zeta_{h}^{k} + \ugamma_{h}^{k} - \hgamma_{h}^{k} \right).
    \end{align}

    By Lemma~\ref{lem:q-ucb}, the regret that we aim to bound is upper bounded by $\sum_{k=1}^{K} \delta_{1}^{k}$.
    Let $n_{h}^{k} := N_{k}(s_{h}^{k}, a_{h}^{k}, b_{h}^{k})$.
    By the regrouping technique in~\citep{jin2018q}, 
    \begin{align*}
        \sum_{k=1}^{K} \sum_{i=1}^{t} \alpha_{t}^{i} \phi_{h+1}^{k^i} 
        & \le \sum_{k'=1}^{K} \phi_{h+1}^{k'} \sum_{t=n_{h}^{k'}}^{\infty} \alpha_{t}^{n_{h}^{k'}} \le (1 + \frac{1}{H}) \sum_{k=1}^{K} \phi_{h+1}^{k}.
    \end{align*}
    Substituting the above into~\eqref{eqn:delta1k} yields 
    \begin{align*}
        \sum_{k=1}^{K} \delta_{1}^{k} \le \sum_{k=1}^{K} \sum_{h=1}^{H} \left(\alpha_{n_{h}^{k}}^{0} H + \frac{1}{H} \phi_{h+1}^{k} + 2\tbeta_{t} + \zeta_{h}^{k} + \ugamma_{h}^{k} - \hgamma_{h}^{k} \right).
    \end{align*}
    Now we bound each term in $\sum_{k=1}^{K} \delta_{1}^{k}$ separately by standard techniques in~\citep{jin2018q, xie2020learning}: 
    \begin{equation}  \label{eqn:sep-bdds}
        \begin{aligned}
            & \sum_{k=1}^{K} \alpha_{n_{h}^{k}}^{0} H 
            \le \sum_{k=1}^{K} H \cdot \I(n_{h}^{k} = 0) 
            \le SABH, \\
            & \sum_{k=1}^{K} \tbeta_{n_{h}^{k}} 
            \le \bigo(1) \sum_{k=1}^{K} \sqrt{\frac{H^3\iota}{n_{h}^{k}}} 
            \le \bigo(\sqrt{H^3SABK\iota}), \\
            & \sum_{k=1}^{K} \sum_{h=1}^{H} (\zeta_{h}^{k} + \ugamma_{h}^{k} - \hgamma_{h}^{k}) 
            = \bigo(\sqrt{H^3 K\iota}) = \bigo(\sqrt{H^3 K \iota}).
        \end{aligned}
    \end{equation}
    Bounding $\frac{1}{H} \sum_{k=1}^{K} \sum_{h=1}^{H} \phi_{h+1}^{k}$ requires additional efforts, since here the relationship $\phi_{h+1}^{k} \le \delta_{h+1}^{k}$ in~\citep{jin2018q} does not necessarily hold.
    Define the martingale difference sequence 
    \begin{align*}
        \gamma_{h}^{k} = \E_{a\sim \mu_{h}^{*}(s_{h}^{k}), b \sim \nu_{h}^{*}} [Q_{h}^{*}(s_{h}^{k}, a, b)] - Q_{h}^{*}(s_{h}^{k}, a_{h}^{k}, b_{h}^{k}).
    \end{align*}
    Then by noting 
    \begin{align*}
        \phi_{h}^{k} 
        = \uQ_{h}^{k}(s_{h}^{k}, a_{h}^{k}, b_{h}^{k}) + \ugamma_{h}^{k} - \oQ{h}(s_{h}^{k}, a_{h}^{k}, b_{h}^{k}) - \gamma_{h}^{k} 
        \le \alpha_{t}^{0} H + \sum_{i=1}^{t} \alpha_{t}^{i} \phi_{h+1}^{k^i} + 2\tbeta_{t} + \ugamma_{h}^{k} - \gamma_{h}^{k},
    \end{align*}
    we obtain 
    \begin{align*}
        \sum_{k=1}^{K} \phi_{h}^{k} 
        \le (1 + \frac{1}{H}) \sum_{k=1}^{K} \phi_{h+1}^{k} + \sum_{k=1}^{K} \left( \alpha_{t}^{0} H + 2\tbeta_{t} + \ugamma_{h}^{k} - \gamma_{h}^{k} \right).
    \end{align*}
    Recursively, for all $h'\in [H]$, 
    \begin{align*}
        \sum_{k=1}^{K} \phi_{h'}^{k} 
        \le (1 + \frac{1}{H})^{H + 1 - h'} \sum_{k=1}^{K} \sum_{h=h'}^{H} \left( \alpha_{t}^{0} H + 2\tbeta_{t} + \ugamma_{h}^{k} - \gamma_{h}^{k} \right).
    \end{align*}
    Then by similar arguments to those in~\eqref{eqn:sep-bdds}, 
    \begin{align}  \label{eqn:sep-bdds-phi}
        \frac{1}{H} \sum_{k=1}^{K} \sum_{h=1}^{H} \phi_{h}^{k} \lesssim SABH^2 + \sqrt{H^5 SABK \iota}.
    \end{align}
    Finally, combining the above separate bounds in~\eqref{eqn:sep-bdds} and~\eqref{eqn:sep-bdds-phi} yields 
    \begin{align*}
        \Regret(K) \lesssim SABH^2 + \sqrt{H^5 SABK \iota}.
    \end{align*}
    
\end{proof}

\end{document}